\documentclass{article}
\usepackage{PRIMEarxiv}
\usepackage[utf8]{inputenc} 
\usepackage[T1]{fontenc}    
\usepackage{hyperref}       
\usepackage{url}            
\usepackage{booktabs}       
\usepackage{amsfonts}       
\usepackage{nicefrac}       
\usepackage{microtype}      
\usepackage{lipsum}
\usepackage{fancyhdr}       
\usepackage{graphicx}       
\usepackage{bbm}
\usepackage{physics}
\usepackage{amsmath}
\usepackage{mathtools}
\usepackage{amssymb}
\usepackage{amsthm}
\graphicspath{{media/}}     
\usepackage{sectsty} 
\sectionfont{\sc}
\subsectionfont{\sc}
\usepackage{algorithm} 
\usepackage{algpseudocode}

\usepackage{caption} 
\usepackage{subcaption}
\usepackage{multirow,xcolor}
\usepackage{rotating}

\title{On Neural Network approximation of ideal adversarial attack and convergence of adversarial training
}

\author{
  Rajdeep Haldar, Qifan Song \\
  Department of Statistics \\
  Purdue University \\
  \texttt{\{rhaldar, qfsong\}@purdue.edu} 
}
\newtheorem{theorem}{Theorem}[section]
\newtheorem{lemma}[theorem]{Lemma}
\newtheorem{corollary}[theorem]{Corollary}
\theoremstyle{definition}
\newtheorem{definition}{Definition}[section]
\theoremstyle{claim}

\theoremstyle{remark}
\newtheorem*{remark}{Remark}
\begin{document}

\maketitle

\newcommand\hol[2]{\mathcal{H}^{#1}_{#2}} 
\newcommand\cubebd{\mathcal{I}_{\partial}} 
\newcommand{\logit}[1]{\frac{1}{1+e^{-#1}}}
\newcommand{\bb}[1]{\mathbbm{#1}} 
\newcommand{\cl}[1]{\mathcal{#1}} 
\newcommand{\pfc}{\cl{F}^p_{\delta}} 
\newcommand{\bpfc}{\tilde{\cl{C}}} 
\newcommand{\sign}{\text{sign}} 
\newcommand{\loss}{\cl{L}(f,\lambda,\cl{Z})}
\newcommand{\gf}{T^{gf}} 
\newcommand{\pgf}[1]{T^{pgf}_{#1}} 
\newcommand{\pgfd}[1]{\bar{T}^{pgf}_{#1}} 
\newcommand{\adv}{x_{adv}}
\newcommand{\bs}[1]{\boldsymbol{#1}} 
\newcommand{\interior}[1]{
  {\kern0pt#1}^{\mathrm{o}}%
}
\newcommand{\phol}{\bar{\cl{H}}(\omega^*,D,\cl{X})} 
\newcommand{\onephol}{\bar{\cl{H}}(\omega,\mathcal{X})} 
\newcommand{\card}{\text{card}}
\newcommand{\qsong}[1]{\textcolor{red}{#1}}
\newcommand{\rhaldar}[1]{\textcolor{blue}{#1}}
\algnewcommand{\Inputs}[1]{%
  \State \textbf{Inputs:}
  \Statex \hspace*{\algorithmicindent}\parbox[t]{.8\linewidth}{\raggedright #1}
}
\algnewcommand{\Initialize}[1]{%
  \State \textbf{Initialize:}
  \Statex \hspace*{\algorithmicindent}\parbox[t]{.8\linewidth}{\raggedright #1}
}

\begin{abstract}
Adversarial attacks are usually expressed in terms of a gradient-based operation on the input data and model, this results in heavy computations every time an attack is generated. In this work, we solidify the idea of representing adversarial attacks as a trainable function, without further gradient computation. We first motivate that the theoretical best attacks, under proper conditions, can be represented as smooth piece-wise functions (piece-wise H\"older functions). Then we obtain an approximation result of such functions by a neural network. Subsequently, we emulate the ideal attack process by a neural network and reduce the adversarial training to a mathematical game between an attack network and a training model (a defense network). We also obtain convergence rates of adversarial loss in terms of the sample size $n$ for adversarial training in such a setting. 
\end{abstract}


\section{Introduction}
\label{sec: Intro}
Neural networks are universal approximators \cite{Hornik1989MultilayerApproximators,Cybenkot1989MathematicsFunction,Leshno1993MultilayerFunction,Pinkus1999ApproximationNetworks} and have enjoyed immense popularity in the realm of machine learning problems and artificial intelligence applications over the past two decades. Their ability to approximate arbitrary functions with high accuracy has facilitated state-of-the-art performance in nonparametric regression and classification problems. 

Although, neural networks showcase high accuracy and fairly well generalization on the training data and testing data; their testing performance is vulnerable to minor perturbations to the original data \cite{SzegedyIntriguingNetworks}. \cite{Goodfellow2014ExplainingExamples} introduced a gradient-based method to generate bounded perturbations of input data (adversarial examples/attack) which significantly hampers the predictive performance of a given neural network. The adversarial robustness of neural networks has been a topic of importance since then. One popular way to make neural networks robust to such attacks is via training over generated adversarial examples (adversarial training). The robustness of the network will directly correspond to the quality or strength of the attack. 

Over the years a slew of work has been done in generating strong adversarial attacks \cite{Bai2021RecentRobustness}, thereby increasing the robustness of the network models when adversarially trained.
Traditionally, the attack generation process has been some sort of gradient-based operation such as the Fast-gradient \cite{Goodfellow2014ExplainingExamples,DongBoostingMomentum}, Iterative gradient \cite{Kurakin2017ADVERSARIALWORLD}, Projected gradient descent (PGD) \cite{Madry2017TowardsAttacks}, Carlini Wagner Attack \cite{CarliniTowardsNetworks} etc, requiring direct access to the analytic form of the model. Attacks that are generated having direct access to the underlying model are known as \textit{white-box attacks}. Alternatively, one can view these attacks as a function of the input data which results in a perturbation in the original space. If we think of attacks as a function, one can approximate it using a neural network. So, the underlying classification/regression model is a neural network that must learn how to defend itself (robustness) against the attacks generated by an adversary neural network; adversarial training in this context then becomes a game (in a game-theoretic sense) between two neural networks. This perspective of using neural networks to generate attacks is fairly recent and has been explored in \cite{Bose2020AdversarialGames,Xiao2018GeneratingNetworks,Wang2019AT-GAN:Examples,Jaeckle2021GeneratingNetworks}. A benefit of using neural networks to generate attacks is that we learn an attack function after training. Once we obtain this attack function spanning over the input space, we can just use it to generate new adversarial examples without having access to the underlying predictive model. Using neural networks to generate attacks we can exhibit so-called \textit{semi-white box} or \textit{black-box attacks}, which have limited access or no access to the underlying model respectively. Another consideration is the computation costs to generate these attacks; The empirical successful white box attacks like PGD or similar iterated gradient methods usually require recurrent gradient computations for each new sample of data. 

Even though a lot of progress has been made on the algorithmic front from both attack \cite{Jaeckle2021GeneratingNetworks,Gondim-Ribeiro2018AdversarialAutoencoders,Wang2019AT-GAN:Examples,DongBoostingMomentum} and defense \cite{Ghosh2018ResistingAutoencoders,Frosst2018DARCCC:Capsules,Qin2020DeflectingAttacks,Schott2018TowardsMNIST} perspectives; there is a lack of theoretical results to justify such methods.
In this paper, we would like to unify these ideas in a common mathematical framework and motivate adversarial training in the context of an adversarial game (attack-defense) between two neural networks.

In this work, we will prove results for RELU-based neural networks defined as:
\begin{definition}[Deep Neural Networks]\label{def:neuralnets}
Define the class of deep Relu networks $f\in \mathcal{F}_{\sigma}(L, W, w, \kappa, B)$ as parametrized functions of the form:
\[f(x)=A^{(L)} \cdot \sigma\left(A^{(L-1)} \cdots \sigma\left(A^{(1)} x+b^{(1)}\right) \cdots+b^{(L-1) }\right)+b^{(L)},\]
where the $A^{(l)}$'s are the weight matrices and $b^{(l)}$'s are the bias vectors with real-valued elements, and the activation function $\sigma: \mathbb{R}\mapsto \mathbb{R}$ is applied element-wise. $\sigma(x)=\operatorname{ReLU}(x)=\max \{0, x\}$ (rectified linear unit). The network is $L$-layer deep and its width is bounded by $\sup_{l} \operatorname{dim}(b^{(l)})\leq w$. Further, we assume that \[\max_{i, j, l}\left|A^{(l)}_{i j}\right| \leq \kappa,\max_{i, l} |b^{(l)}_i| \leq \kappa, \sum_{l=1}^{L}\left\|A^{(l)}\right\|_{0}+\left\|b^{(l)}\right\|_{0} \leq W, \text { for } i=1, \ldots, L,\]
i.e. all elements in the $A^{(l)}$'s and $b^{(l)}$'s are bounded in absolute value by $\kappa$, and there are at most $W$ non-zero parameters in total. Finally, we assume $\|f\|_\infty \leq B<\infty$ for all $f$. If the particular value $B$ is an arbitrarily large constant, we may suppress the notation and write $\mathcal{F}_{\sigma}(L, W, w, \kappa, B)=\mathcal{F}_{\sigma}(L, W, w, \kappa)$.
\end{definition}
\subsection{Contribution}
  We first formalize the notion of an ideal or best possible attack characterized by the underlying loss for the classification/regression problem. Then we justify how gradient-based PGD and FGSM  attacks are related to the ideal attack. Furthermore, we show indeed the ideal attack can be approximated by neural networks up to arbitrary accuracy; justifying the use of neural networks to generate attacks. Subsequently, we derive convergence bounds for adversarial training in the context of an adversarial game between two neural networks as the sample-size $n$ increases.
Finally, we showcase some numerical experiments to verify the capabilities of neural networks to learn the ideal attack function. Also, the robustness and attack strength of adversarial trained networks in the context of an adversarial game has been compared to the PGD robustly trained networks and PGD attacks respectively.

\subsection{Problem definition / Notation}
Let $(x,y)=\cl{Z}\in \cl{X}\times \cl{Y}=\cl{\tilde{Y}}$ represent an observation pair where $x\in \mathcal{X}$ is the predictor and $y\in \mathcal{Y}$ is the response, where $\mathcal{Y}=\mathbbm{R}^{k}$ or $\mathcal{Y}=\{1,\dots,k\}$  for regression and classification problems respectively. Without loss of generality, we assign $\mathcal{X}= [0,1]^D$ after normalization.

Let $f :\mathcal{X}\to\mathbbm{R}^k\in \mathcal{C}$ be our model, where $\mathcal{C}$ is some function class. $\cl{L}: \mathbbm{R}^k\times\mathcal{Y}\to \mathbbm{R}$ denotes a loss function which we wish to minimize, this could be the negative log-likelihood or any reasonable discrepancy measure between $f(x)$ and $y$. An ideal model minimizes the expected loss, i.e., $f^*=\arg\inf\limits_{f\in\mathcal{C}} \mathbbm{E}\cl{L}(f(x),y)$, where $\bb{E}=\bb{E}_{(x,y)\sim \cl{P}}$ denotes the expectation over the true distribution $\cl{P}$ of the data.

\begin{definition}[$(p,\delta)$ Perturbation function class]
$\pfc (\bpfc)$ is a $(p,\delta)$ perturbation function class restricted on function class $\bpfc$
\begin{equation}
    \pfc = \{g:\cl{\tilde{Y}}\to\cl{X}\in \bpfc: \forall (x,y)\in \cl{\tilde{Y}} ,\,\norm{g(x,y)}_p\leq \delta \},
\end{equation}
where $\bpfc$ is some base function class, for instance, the class of continuous functions, spline functions, H\"older functions, etc; 
\label{def:perturbation_function_class}
\end{definition}

\begin{definition}[$\ell_p(\delta)$ function attack]
$\lambda^f$ is the best $\ell_p(\delta)$ function attack (within the family $\cl{F}^p_{\delta}(\bpfc)$) on $f$ iff
\begin{equation}
    \lambda^f(x,y)=\arg\max\limits_{\lambda \in \cl{F}^p_{\delta}(\bpfc)}\mathbbm{E}\cl{L}(f(x+\lambda(x,y)),y).
    \label{best_function_attack}
\end{equation}
\end{definition}
Our definition for $\ell_p(\delta)$ function attack serves as a surrogate of the {\it point wise} $\ell_p$ attacks of strength $\delta$ in the literature, defined as
\begin{equation}
    \lambda^{p.w}(x,y)=\arg\max_{\norm{\lambda}_p\leq \delta} \cl{L}(f(x+\lambda),y).
    \label{pointwise_attack}
\end{equation}

Note that if $\cl{F}^p_{\delta}(\bpfc)$ has a nice parameterized form, then the evaluation of $\lambda^f$, e.g., designing an attack for new observations,  would be more computationally efficient than $\lambda^{p.w}(x,y)$, which generally needs to be evaluated point-wisely. Our framework is similar to the one defined in \cite{Bose2020AdversarialGames}. 


\subsubsection*{Adversarial Training}
The adversarial training aims to make the model robust against adversarial attacks via minimizing the objective loss under the best possible attack function on the model, i.e.,
\begin{definition}[Robust $(p,\delta)$ model]
$f^*$ is a  robust $(p,\delta)$ model if it minimises the following objective:
\begin{equation}
    f^*=\arg\min\limits_{f\in \cl{C}}\max\limits_{\lambda \in \cl{F}^p_{\delta}(\bpfc)}\mathbbm{E}\cl{L}(f(x+\lambda(x,y)),y)
    \label{adversarial_training}
\end{equation}
where $\mathbbm{E}\cl{L}(f(x+\lambda(x,y)),y)$ is the adversarial loss.
\end{definition}

When shifting from the point-wise paradigm to a functional form of an attack, the natural question is which perturbation function class $\pfc(\bpfc)$ (or more specifically the choice of $\bpfc$) is appropriate to induce adversarial perturbations, that can approximate the strongest point-wise attack $\lambda^{p.w}$, of any given input $(x,y)$ and model $f\in\mathcal C$?


The theoretical goal of this work is to determine a function class that includes the optimal pointwise attack (i.e., $\lambda^{p.w}\in \bpfc$), then we can try to approximate this function class up to arbitrary accuracy using a parametric family, i.e., neural networks. Subsequently, we can then view our adversarial training framework \eqref{adversarial_training} as a mathematical game between two networks, and establish the statistical properties for this game.

\section{Best adversarial attack}
In this section, we will first develop some heuristics and then rigorously show that for smooth $f$'s, the point-wise attacks \eqref{pointwise_attack} can be represented by functions of a certain form. We will first develop some intuition by analyzing the simple case of $\ell_{\infty}(\delta)$ attacks on linear and logistic regression models.

For the sake of brevity, from now onwards we will denote:
\begin{equation}
    \cl{L}(f(x+\lambda(x,y)),y)\coloneqq \cl{L}(f,\lambda,\cl{Z})
\text{ and } \cl{L}(f(x),y)\coloneqq \cl{L}(f,\cl{Z}) 
\label{def:notation_conv_loss}
\end{equation}
\subsection{Attacks on linear models and FGSM}
\label{subsec:linearmods}
\subsubsection*{Linear regression}
Let $y|x\sim N(\beta\cdot x, \sigma^2)$ be the true relationship, and $f(x)=\hat{\beta}\cdot x$ be the estimated model. The loss, in this case, is the negative log-likelihood up to a constant: $\cl{L}(f(x),y)=(y-\hat{\beta}\cdot x)^2$. For a given $\hat{\beta}$, the ideal function attack $\lambda^f$ should satisfy the objective: $\lambda^f=\arg\max_{\lambda\in\pfc(\bpfc)}\mathbbm{E}\cl{L}(f(x+\lambda),y)$. With some straightforward one can see that the \textit{point wise} attack \eqref{pointwise_attack} is as follows:
\begin{align*}
   \lambda^{p.w}(x,y) & =\delta \cdot \sign [(\hat{\beta}\cdot x-y)/\hat{\beta}].
\end{align*}
Note that one easily derives that this attack is equivalent to $\lambda^{p.w}(x,y)=\delta \cdot\sign(\nabla_x \cl{L}(f,\cl{Z}))$ as $\sign(\frac{1}{\hat{\beta}})=\sign(\hat{\beta})$.
For the multivariate case $Y|X\sim N(\beta^TX,\Sigma)$ where $\beta$ is a vector now, one can get analogous results. The essential point is that for the linear regression case, the best point-wise adversarial attacks are essentially \textit{piece-wise constant functions} which are related to the gradient of the underlying loss $\cl{L}$ w.r.t. $x$.
On the other hand, if our perturbation function class $\pfc(\bpfc)$ is flexible enough and can well approximate any piece-wise constant functions, then trivially, $\lambda^f(x,y)\approx\lambda^{p.w}(x,y)$.
\subsubsection*{Logistic regression}

For binary classification, $\cl{Z}=(y,x)\in \{0,1\}\times [0,1]$ follows the conditional distribution: $\bb{P}(y=1|x)=p(\beta,x)=\logit{\beta x}$, i.e., $y|x\sim {\rm Bernoulli}(p(\beta,x))$. Let the model $f$ be the logistic regression fit with parameter estimator $\hat{\beta}$. The loss is the negative log-likelihood up to a constant: $\cl{L}(f(x),y)=-y\ln(p(\hat{\beta},x))-(1-y)\ln(1-p(\hat{\beta},x))$. As in linear regression case, one can calculate $\lambda^{p.w}$; by simple algebra, one can see the critical condition is $\hat{\beta}(p(\hat{\beta},x+\lambda)-y)=0$ and $p(\hat{\beta},x+\lambda)$ is monotonic in $\lambda$, which results in
\begin{equation}\lambda^{p.w}(x,y)=\delta\cdot \sign(0.5-y)\sign(\hat{\beta}).\end{equation}
We can replicate the exact same arguments for the multivariate case, and get the analogous result (with $\hat{\beta},x$ now being vectors).
As in linear regression case, one can express $\lambda^{p.w}$ (under the $\ell_{\infty}(\delta)$ function attack) in terms of the gradient of $\cl{L}(f,\cl{Z})$, and we get the exact same form: 
\begin{equation}
    \lambda^{p.w}(x,y)=\delta \cdot\sign(\nabla_x \cl{L}(f,\cl{Z})).
\end{equation}
Consequently, if the function class $\pfc(\bpfc)$ can well approximate any piece-wise constant functions, then $\lambda^f(x,y)\approx\lambda^{p.w}(x,y)$.
\begin{remark}[FGSM as $\ell_{\infty}(\delta)$ function attack]
The attack of form $\delta \cdot\sign(\nabla_x \cl{L}(f,\cl{Z}))$ is known as the fast-gradient signed max (FGSM) attack in the literature. \cite{Goodfellow2014ExplainingExamples} argues that FGSM is the exact attack for logistic regression. We can reaffirm this even from a function attack perspective, if the function class $\pfc(\bpfc)$ is sufficiently broad. In general, we can extend this idea and think of FGSM as a piece-wise constant function which is the best   $\ell_{\infty}(\delta)$ function attack for linear models. The idea of a piece-wise smooth function for arbitrary loss is a much more general concept and we will see how FGSM fits as a special case in our framework.
\end{remark}

\subsection{Gradient flow-based attack}
\subsubsection*{Basic concepts}
In this section, we will construct a functional form for the best point-wise $\ell_p(\delta)$ attack using the concept of gradient flows and continuous dynamic systems (CDS). For the sake of completeness, preliminary knowledge of CDS can be found in the appendix. Readers of interest in the detailed background and theories of CDS may refer to Katok's book on the Modern Theory of dynamical systems or Brin's book which provides a more condensed version of the theory \cite{Brin2002IntroductionSystems,Katok1995IntroductionSystems}. Definition \ref{cds} formally defines a CDS $T$, as a function of time in an arbitrary metric space $\cl{X}$. $T$ can be thought of as a law that dictates the movement of a particle in $\cl{X}$  as a function of its initial position, i.e., $T(t)x_0$ is the location of a particle at time $t$ initializing from $x_0$. Given a CDS $T$, an object of interest is the set of points in space $\cl{X}$ where the particle is headed given any initial position. These special points can be said to \textit{attract} the particle towards them (i.e. Definition \ref{attraction}). We would also want the trajectory of the particle given a particular initial point to be tractable, and the set of initial points $\cl{A}(T)$ for which the trajectory is well defined is called the \textit{attractor} (Definition \ref{globalattractor}).

A point $x_0\in \cl{X}$ is a stationary point for a dynamical system $T$
if $T(t)x_0 = x_0$ for all $t \in \bb{R}^+$ (i.e., a stationary point is a singleton invariant set \ref{invarient set}). Corresponding to a given CDS $T$, one can define $T^{-1}$ which represents the inverse dynamics, i.e. for any $s\geq t$ and $s,t \in \bb{R}^{-}$ we have $T(s-t)T^{-1}(t)=T^{-1}(s)$. If $T^{-1}$ is unique for a $T$, we say $T$ is one-one. $T^{-1}$ controls the flow of particles going back in time, one can think of it as an inverse mapping. If $T$ is one-one, we can have an abuse of notation and extend $T$ to the negative time domain as $T(s)\coloneqq T^{-1}(s) \text{ if } s<0$.

We can also define the concept of stable ($W_s(x_0)$) and unstable set ($W_u(x_0)$)  of a stationary point $x_0$ (Definition \ref{stable/unstable_sets}) of an one-to-one CDS. The stable set is essentially the set of points that eventually reach $x_0$ under the dynamic system $T$; Similarly, the unstable set is the set of points that eventually reach $x_0$ under the inverse dynamics $T^{-1}$.

\subsubsection{Gradient Flow}
\label{subsubsec:GF}
We will consider a CDS $T^{gf}$, which follows a gradient dynamic, i.e., $T^{gf}$ is  the solution to the ordinary differential equation $\dot{x}=\nabla_xF(x)$ where $F(x)=\cl{L}(f(x),y)$. Note that for the simplicity of the technical analysis, we assume that $F(x)$ is continuously differentiable and all of its local optimums locate in the interior of $\mathcal X$.

Assume that $\mathcal X$ is a compact space, if the gradient is bounded everywhere in $\mathcal{X}$, then $\cl{A}(T^{gf})=\cl{X}$. Also, the stationary points of $\gf$ ($\{x:\gf(t)x=x;\, \forall t\in \bb{R}^+\}$) are essentially the equillibria/critical points ($E=\{x: \nabla{F(x)=0}\}$) of $F$. A classic result in the literature is that for $\gf$ the set of equillibria \textit{attracts} $\cl{A}(\gf)$  (Section 6, Theorem 6.12 \cite{Hale2004StabilitySystems}). This essentially means that if the particle follows the gradient dynamic system, it is guaranteed to reach a critical point.

The gradient flow satisfies a lot of nice properties including point-dissipative (Definition \ref{point_dissipative}) set of equilibria, being one-one CDS, $(T^{gf})^{-1}$ being just the negative gradient flow, i.e., solution of the ODE $\dot{x}=-\nabla F(x)$, etc.

Also, for a maximization framework, the path followed by gradient flow is a natural choice to study as it is locally maximal; following the gradient flow path, the function value is non-decreasing as $\dot{F}(x)=\norm{\nabla{F(x)}}_2^2\geq0$ .

\subsubsection{Decomposition of $\cl{X}$ into manifolds}
\label{subsubsec:decomposition_into_manifolds}
For $\gf$, the critical points (w.r.t. gradient) are essentially the stationary points (w.r.t. dynamic system)and the stable set $W_s(x_0)$ for any $x_0\in E$ are essentially the set of initial positions which lead to the local maxima/saddle $x_0$ when following the gradient path. Intuitively any point in our metric space $\cl{X}=\cl{A}(\gf)$ should lead to one of these critical points under $\gf$. Remark 6.11 of \cite{Hale2004StabilitySystems} formalizes this notion and we get $\cl{X}=\bigsqcup_{x_0\in E}W_s(x_0)$.

Essentially we are able to partition the metric space based on the local maxima/saddle points of the gradient flow dynamic system.
This decomposition will be useful when we argue the piece-wise nature of our function attack.
Additionally, if $\gf(t)x$ is $C^k$ w.r.t. $t$ for any $x\in W_s(x_0)$ then $W_s(x_0)$ is a $C^k$ manifold, and $\cl{X}$ is just union of manifolds (see; Section 4 of \cite{Hale2004StabilitySystems}). 

\subsubsection{Projected gradient flow}
\label{subsubsec:pgf}
Let's denote $x_{adv}=\lambda^{p.w}(x_s,y)+x_s$. 
By the definition of a point-wise $\ell_p$ attack of strength $\delta$ \eqref{pointwise_attack}, $\adv$ must satisfy the following constraint optimisation problem $\forall x_s \in \cl{X}$:
\begin{align}
    &\arg\max_{x\in\cl{X}} F(x), \nonumber\\
    &\text{subject to } g_i(x)\leq 0,\quad \forall i \in \{1,\dots,Q\}
    \label{contraint_optimisation}
\end{align}
for some $g_i$'s, where $g_i(x)$ represents the inequality constraints and $Q$ is the number of constraints. The $i^{th}$ constraint is said to be active if and only if $g_i(x)=0$. To be more specific, for $1\leq p < \infty$, we have $Q=1$ and  $g_1(x)=\norm{x-x_s}_p^p-\delta^p$; for $p=\infty$, we have $Q=D$ and $g_i(x)=\abs{(x-x_s)_i}-\delta$ for all $i\in \{1,\dots, D\}$. The constraints are active in this case when $x\in \partial B_{\ell_p(x_s,\delta)}$ (boundary of the $\ell_p$ ball of radius $\delta$ around $x_s$).

Let $\boldsymbol{J}(x)=(\dots,\bb{1}_{\{g_i(x)=0\}},\dots)^T$ be a $Q\times 1$ vector whose $i^{th}$ component represents whether the $i^{th}$ inequality constraint is active or not (1 or 0). $\boldsymbol{J}(x)$ induces an indexing set $I(x)\subset \{1,\dots,Q\}$ for the active constraints. 
Define $\bs{C}(x)=[(g_i(x))_{i\in I(x)}]$ to be the active constraint matrix and $D\bs{C}(x)=[(\nabla g_i(x))_{i\in I(x)}]$ be its corresponding derivative/sub-derivative matrix. A projection operator which takes any vector to the surface defined by the active constraints (i.e., $\{x: g_i(x)=0,\, \forall i\in I(x)\}$) can be canonically defined as $$\bs{P}(x)=\bs{I}-D\bs{C}(x)^T(D\bs{C}(x)D\bs{C}(x)^T)^{-1}D\bs{C}(x).$$
Note when all constraints are inactive, $\bs{C}(x)$ is a null vector and the projection reduces to the identity $\bs{P}(x)=\bs{I}$.

Consider a CDS $\pgf{x_s}$ on $\mathcal X=B_p(x_s,\delta)$ parameterized by $x_s$ as a solution to the ODE:
\begin{equation}
\dot{x}=\bs{P}(x)\nabla_x F(x)
    \label{progected_gradient_flow}.
\end{equation}
This can be viewed as a special case of the constraint optimization problems discussed by \cite{Feppon2020NullOptimization}.

Note that for $1\leq p < \infty$ as we have a single norm constraint $g_1(x)=\norm{x-x_s}_p^p-\delta^p$, $\bs{J}(x)=\bb{1}_{g_1(x)=0}$ and 
\begin{equation}
    \bs{P}(x)=\bs{I}-\bs{J}(x)(\norm{\nabla g_1(x)}^2_2)^{-1}\nabla g_1(x)\nabla g_1(x)^T.
    \label{eqn:lp_notinf_ODE}
\end{equation}
When the particle is at the boundary of the $\ell_p$ ball $B_p(x_s,\delta)$, i.e., $\bs{J}(x)=1$,  the projection matrix is just orthogonal to the normal vector to the surface of the $\ell_p$ ball. thus, when the particle is at the boundary it doesn't go beyond and explores the boundary surface based on the gradient direction over the surface. The dynamic stops once the gradient becomes parallel to the normal vector of the surface, in other words, there is no direction left to explore which increases the function $F$ further (on the surface). When the particle is inside the $\ell_p$ ball it just follows the standard gradient direction. 

For $p=\infty$, the $\ell_p$ ball is just a hyper-cube and $\bs{J}(x)\neq\bs{0}$ indicates which sides/diagonal of the cube the particle resides in. The ODE in this case is just:
\begin{equation}
    \dot{x}=\nabla_x F(x)\odot(\bs{1}-\bs{J}(x)),
    \label{eqn:l_inf_ODE}
\end{equation}
where $\odot$ is the element-wise product. The ODE for $p=\infty$ just says that once we reach the boundary of the cube, we move only along the sides by ignoring gradient components where the particle has attained side length $\delta$.
The path traversed by $\pgf{x}$ is also non-decreasing as $\dot{F}(x)=\nabla F^T(x)\bs{P}\nabla F(x)\geq 0$, as $\bs{P}$ is  positive semi-definite being a projection.

\subsubsection{Deflecting saddle points}
\label{subsubsec:deflect}
One issue with $\pgf{x_s}$ is that it can't distinguish between saddle points (where $\nabla^2_x F(x)$ hessian is indefinite) and local maxima; the continuous dynamic stops as long as (i) $\nabla F(x)=0$ or (ii) $x$ on the surface of the $\ell_p$ ball and $\nabla F(x)$ is parallel to the normal vector of the $\ell_p$ ball. We can modify \eqref{progected_gradient_flow} to traverse a dynamics which doesn't get stuck at the saddle points. Let $S=\{\eta:\nabla^2_x F(x)|_{x=\eta} \text{ is indefinite} \}\subseteq E=\{x:\nabla{F(x)=0}\}$ be the set of saddle points. We can add a deterministic perturbation towards the unit eigenvector $\nu^\eta_{max}$ corresponding to the maximum eigenvalue of $\nabla^2_x F(x)|_{x=\eta}$ to exclude the saddle points from being a stationary point. Therefore, we define a CDS $\pgfd{x_s}$ on $B(x_s,\delta)$ parameterized by $x_s$ as a solution to the ODE:
\begin{equation}
\dot{x}=\bs{P}(x)[\nabla F(x)+\sum_{\eta\in S}\psi_{\epsilon}^{\eta}(x)\nu^{\eta}_{max}]
    \label{projected_gf_excluding_saddle_pts},
\end{equation}
where $\epsilon>0$ and $\psi_{\epsilon}^{\eta}(x)=\exp{-\frac{1}{\epsilon^2-\norm{x-\eta}_2^2}}$ if $ x\in {B}_2(\eta,\epsilon)$ 
and $\psi_{\epsilon}^{\eta}(x)=0$ elsewhere. Essentially we are adding a smooth perturbation to the originally defined projected gradient flow dynamics as it gets closer (within $\epsilon$ radius) to any saddle point (deflection). 
\begin{remark}
Our goal is that $\pgfd{x_s}$ can retain all the local smoothness, local optimum, and non-decreasing properties of $\pgf{x_s}$ with a proper choice of $\epsilon$.
Under certain regularity of $F$ (see; Theorem \ref{thm:att}), there should exist some sufficiently small $\epsilon$, such that
$\psi_{\epsilon}^{\eta}(x)$ is a small-bump function, which is smooth and non-negative everywhere. Near any saddle point $\eta\in S$, we have $\langle \nu^{\eta}_{max},\nabla{F}(x)\rangle\geq0$, which preserves the non-decreasing trajectory of the original ODE. At a local maximum, all the eigenvalues of $\nabla^2_xF(x)$ are $<0$, resulting in the ODE around local maxima having no deflection term. Equilibrium is attained as long as $\nabla F(x)=0$, preserving the local maximas from the original ODE.  
\end{remark}

We can now define an attack function as:
\begin{equation}
\lambda^{pgf}(x,y)=\lim_{t\to\infty}\pgfd{x}(t)x-x
    \label{gradient_flow_attack}.
\end{equation}

Our next theorem demonstrates that \eqref{gradient_flow_attack} is the \textit{best point-wise attack}. More importantly, by the nature of CDS, $\lambda^{pgf}(x,y)$ has certain regularity (e.g., local smoothness) properties, i.e., it belongs to some \textit{perturbation function class} $\pfc$ consisting of regulated functions, therefore, $\lambda^{pgf}(x,y)$ is also the \textit{ best function attack} over this particular $\pfc$. Details of this $\pfc$ will be addressed in section \ref{subsec:motivation_piece}.

\begin{theorem}\label{thm:att}
Given $F(x)=\cl{L}(f(x),y)$ be locally $C^2$ almost surely, and  the set of critical points $E=\{x:\nabla{F(x)=0}\}$ be finite, non-degenerate with any $a,b\in E$ satisfying $dist(a,b)\geq\delta$; choosing any $0<\epsilon<\delta$ in \eqref{projected_gf_excluding_saddle_pts}, then $\lambda^{pgf}(x,y)$ in \eqref{gradient_flow_attack} is 
the best \textit{point-wise} attack of strength $\delta$.
\label{thm:bestattack}
\end{theorem}
The proof relies on showing that $\adv=\lambda^{pgf}(x,y)+x$ satisfies the KKT (Karush–Kuhn–Tucker) conditions and the SOSC (second order sufficient condition) for the constraint optimization \eqref{contraint_optimisation}.

\subsection{Discretization of the best attack}
\label{subsec:discrete_best_atk}
Equation \eqref{gradient_flow_attack} is a construction of an ideal/best $\ell_p(\delta)$ point-wise attack function for a given loss $F(x)=\cl{L}(f(x),y)$. The underlying CDS $\pgfd{x}$ used in construction for this attack utilizes the information related to the saddle points of $F(x)$ (\ref{projected_gf_excluding_saddle_pts}). In practice, we usually have no access to the saddle points or the eigenvalues associated with them; that itself is a challenging optimization problem especially when $f(x)$ is characterized by deep neural networks.

Note that the terms $\bs{P}(x)\sum_{\eta\in S}[\psi_{\epsilon}^{\eta}(x)\nu^{\eta}_{max}]$ were added to the original projected gradient flow dynamics $\pgf{x}$ \eqref{progected_gradient_flow} to account for deflection around saddle points, so that the dynamic doesn't get stuck in one of such points. For our purpose, we were interested in showing the existence of an ideal attack using such artificial construction. The key point was that such a dynamic can locally solve the unconstrained optimization problem \eqref{contraint_optimisation}. In real life, one can achieve an equivalent effect with very \textit{high probability} by introducing stochastic noise to the original projected gradient flow dynamics \eqref{progected_gradient_flow}. This leads to an ODE of the underlying dynamics being a stochastic differential equation: $d x(t)=\bs{P}(x(t))[\nabla F(x(t)) +2d B(t)]$ (where $B(t)$ is a Brownian motion). It is well known that introducing stochasticity in the original dynamics, helps us solve the issue of getting stuck in saddle points \textit{almost surely}. This fact has been illustrated in \cite{Du2017GradientPoints, Neelakantan2015AddingNetworks,Yang2017OnFlow,Jin2017HowEfficiently}.

So the core component dictating the dynamic system is the term $\dot{x}=\bs{P}(x)\nabla F(x)$ accompanied by a deflection term or noise term to avoid saddle points. If we consider the discrete time step version of the ODE based on the core component we get: \begin{equation}
    x^{t+1}=\Pi(x^{t}+\gamma \nabla F(x^t)),
    \label{pgdl2}
\end{equation} where $\gamma$ is the step-size and  the operator $\Pi$ is (discrete) analogous to $\bs{P}(x)$ which is a  (continuous) projection to the $\ell_p$ ball of radius $\delta$ around $x^0$ (initial point). 
This is a projected gradient-ascent algorithm derived from the underlying ODE \eqref{progected_gradient_flow}. 
Note that to eliminate the possibility that the process (\ref{pgdl2}) is stuck at a saddle point, one can always inject noise into it by replacing $\nabla F$ with its stochastic version (e.g., stochastic gradient descent).
In general, the discretization can be characterized by any projected- steepest ascent algorithm of the form:
\begin{equation}
x^{t+1}=\Pi(x^{t}+\gamma \vec{d}),
\label{steep_ascent}
\end{equation}
where $\vec{d}$ is the steepest ascent direction. The steepest ascent direction $\vec{d}$ can be attributed as a unit vector with respect to $\norm{}_p$ norm which is maximal along the gradient direction. \cite{Balles2020TheDescent} explores and gives justification for these various steepest ascent directions. Steepest ascent w.r.t. $\norm{\cdot}_2$ norm gives $\vec{d}\propto\nabla F$ which reduces to \eqref{pgdl2}. Steepest ascent w.r.t. $\norm{\cdot}_1$ gives us $\vec{d}\propto \nabla F_{i_{max}}\cdot\bs{e}^{i_{max}}$ which is essentially a coordinate ascent update step ($i_{max}=\arg\max_i (\nabla F)_i$, $\bs{e}^i$ is the $i^{th}$ standard basis vector in $\bb{R}^D$). Steepest ascent w.r.t. $\norm{}_{\infty}$ implies $\vec{d}\propto \sign(\nabla F)$ which is a signed gradient update step. In $\bb{R}^D$, all the norms are equivalent and all of these steepest ascent algorithms will generally enjoy the same convergence rate \cite{Balles2020TheDescent}. 

\begin{remark}[PGD and FGSM attack as discretization of the best attack]
If we consider the steepest ascent direction w.r.t. $\norm{\cdot}_{\infty}$, $\vec{d}\propto\sign(\nabla F)$, then the discrete time step updates for our ODE \eqref{progected_gradient_flow} becomes:
\begin{equation}
    x^{t+1}=\Pi(x^{t}+\gamma \,\sign\nabla F(x^t)).
    \label{pgd_attack}
\end{equation}
This is essentially the projected gradient descent attack proposed in the eminent paper \cite{Madry2017TowardsAttacks}. As $\gamma\to 0$ the discrete dynamics converge to the continuous dynamics. In real life, as the PGD updates are done on batches of data, the discrete dynamics inherently include a noise term due to random batch selection of data, which helps us deflect the saddle points. 

Also, when the step size $\gamma=\delta$, then there is essentially only one update step $t=0$ and we get $x=x^0+\delta\,\sign \nabla F(x^0)$. This is equivalent to the FGSM attack proposed in \cite{Goodfellow2014ExplainingExamples}. We can now understand why the FGSM turns out to be the best $\ell_p(\delta)$ function attack for linear models \ref{subsec:linearmods}; for linear models, the FGSM is exact and equivalent to \eqref{best_function_attack} as the gradient direction is constant.
\end{remark}

\subsection{Motivation for piece-wise H\"older spaces}
\label{subsec:motivation_piece}
Theorem \ref{thm:bestattack} guarantees the existence of a $\ell_p(\delta)$ function attack, now the natural next step is to define a perturbation function class to which such a function attack belongs. Once we define such a function class, we aim to show it can be approximated by neural networks with arbitrary accuracy and as a consequence, we can emulate the best function attack \eqref{gradient_flow_attack} using neural networks.

In  section \ref{subsubsec:decomposition_into_manifolds}, it was discussed that $\gf$ induces a partition of the underlying metric space into stable sets $\cl{X}=\bigsqcup_{x_0\in E}W_s(x_0)$. If $\gf$ is $C^k$ smooth in $W_s(x_0)$ w.r.t. $x$, then this will induce smoothness in $\pgfd{x}$ \eqref{projected_gf_excluding_saddle_pts} and $\lambda^f(x,y)$ defined in \eqref{best_function_attack} up to certain singularities.
$\lambda^{pgf}(x,y)$ can be seen as local projection of $\gf$ characterised by $\bs{P}(x)$, which inherits smoothness properties of $\gf$ almost everywhere.
\begin{lemma}
    \label{lem:motivation piece wise holder}
    Let $\nabla F$ be $C^k$ smooth w.r.t. $x\in \bigsqcup_{x_0\in E}W^{\mathrm{o}}_s(x_0)$ then,
    $\lambda^{pgf}(x,y)$ is $C^k$ smooth w.r.t. $x\in\bigsqcup_{x_0\in E}W^{\mathrm{o}}_s(x_0)$ for all $p$, where $W^{\mathrm{o}}$ denotes the interior of the set $W$.
\end{lemma}

 As we are dealing with compact spaces ($x\in[0,1]^D$), $C^k$ smoothness will imply belonging to a H\"older space (Definition \ref{def:holder}) w.r.t. $x$. We can consider the space of piece-wise H\"older function space to describe $\lambda^f(x,y)$ where the pieces are characterized by the singularities (Boundaries $\partial W_s(x_0); x_0\in E$).

\label{sec:headings}
\section{Piece-wise H\"older function and its approximation}

\begin{definition}
Let $\bs{A}=\{A_i\}$ be any disjoint partition of the domain space $\cal{X}$. i.e, $\mathcal{X}=\bigsqcup_{i}A_i$. Then a function $h:\mathbb{R}^D\to \mathbb{R}$ belongs to $\bar{\cl{H}}(\omega,\mathcal{X})$, a piece-wise H\"{o}lder function class iff for some $\bs{A}$, the restriction of the function $h$ to any piece $A_i\in \bs{A}$, denoted by $h_{A_i}$, belongs to $\cl{H}(\omega,A_i)$ (Definition \ref{def:holder}).
\end{definition}

Based on Section \ref{subsec:motivation_piece}, we would like to consider the perturbation function class (Definition \ref{def:perturbation_function_class}) as $\lambda^f(x,y)\in\pfc(\tilde{\cl{C}})=\pfc(\bigotimes_{i=1}^D\bar{\cl{H}}(\omega_i,\mathcal{X}))$ (where $\bigotimes$ denotes the product space). For the sake of notational convenience we will denote $\bar{\cl{H}}(\omega^*,D,\cl{X})\coloneqq\bigotimes_{i=1}^D\bar{\cl{H}}(\omega_i,\mathcal{X})$ where $\omega^*=\min \{\omega_i\}$.

There has been a convincing argument to use $\tilde{\cl{C}}=\phol$, but the way we defined $\lambda^f(x,y)$\eqref{gradient_flow_attack} won't have a closed-form expression. Hence, we would like to use some closed-form parametric family of functions to emulate the perturbation function class where the \textit{best attack function} resides. 

We can use a family of Relu neural networks (Definition \ref{def:neuralnets}) to approximate our perturbation function class with arbitrary accuracy up to a small measurable set $\nu$. Theorem \ref{thm:NNapprox} shows that using a sufficiently wide network one can achieve arbitrarily small approximation error.
\begin{definition}[Asymptotic notation]
    For any two real functions a,b:
    \begin{itemize}
        \item $a\prec b$ or $a=\mathcal{O}(b)$ denotes there exists a positive constant $0<C<\infty$ such that $\abs{a}\leq C\abs{b}$.
        \item $a\succ b$ or $a=\Omega(b)$ denotes there exists a positive constant $0<C<\infty$ such that $\abs{a}\geq C\abs{b}$
        \item $a\asymp b$ or $a=\Theta(b)$ iff $a\prec b \prec a$.
        \item Consider two real sequences $a_n, b_n$ then $a_n=o_\mathbb{P} (b_n)$ denotes that $\frac{a_n}{b_n}$ converges to 0 in probability.
    \end{itemize}
\end{definition}
\begin{theorem}
Let $\mu$ be a probability measure on the space $\cl{X}$. For any $h\in \phol$ and $\cl{X}_s=Supp(\mu)\subset \cl{X}$ with $\dim_M(\cl{X}_s)=d$ (see; Definition \ref{def:minkowski}). Then given $W\geq c_1 \epsilon^{-\frac{d}{\omega^*}},w\geq  c_2 \epsilon^{-\frac{d}{\omega^*}},\kappa\geq c_3\epsilon^{-c_4}, L=c_5$ and a measurable set $\nu$ with $\mu(\nu)=\cl{O}(\epsilon^{\frac{1}{\omega^*}})$, there exists a neural network $\hat{h}$ such that:
\begin{equation}
    \inf_{\hat{h}\in \cl{F}_{\sigma}(L,W,w,\kappa)}\sup_{x\in \cl{X}_s\setminus \nu}||\hat{h}(x)-h(x)||_{\infty}\leq \epsilon,
\end{equation}
where $c_1,c_2,c_3,c_4,c_5>0$ are some large enough constants independent of $\epsilon$ .
\label{thm:NNapprox}
\end{theorem}

\begin{corollary}
For any $h\in\phol$ with $\mathcal{X}_s=Supp(\mu)\subset \cl{X}\,, \dim_M(\cl{X}_s)=d$, there exists $\hat{h}\in \cl{F}_{\sigma}(L,W,w,\kappa)$ with sufficiently large $\kappa$ and depth $L$ such that:
\begin{equation}
    \sup_{x\in\cl{X}\setminus \nu}\norm{\hat{h}(x)-h(x)}_{\infty}=\cl{O}\left(\min(W, w)^{-\omega^*/d}\right),
\end{equation}
where $\mu(\nu)=\cl{O}\left(\min(W,w)^{-1/d}\right)$.
Furthermore, if $h\in\bigotimes_{i}\cl{H}({\omega}_i,\cl{X})$, then $\mu(\nu)=0$.

\label{cor:NN-approx}
\end{corollary}
\begin{proof}
Representing $\epsilon$ in terms of the width $w$ and number of non-zero parameters $W$ we get the above result. Also, when $h\in\prod_{i}\cl{H}({\omega}_i,\cl{X})$ is just product space of H\"older functions without any pieces, we don't require the measurable set $\nu$ to bound the error. As in the proof of Theorem \ref{thm:NNapprox},  $\nu$ arises as a consequence of covering the boundary of pieces in a piece-wise H\"older function.
\end{proof}

Theorem \ref{thm:NNapprox} and corollary \ref{cor:NN-approx} justifies the existence of a Relu neural net which can approximate the best $\ell_p(\delta)$ function attack $\lambda^f(x,y)$ as defined in \eqref{gradient_flow_attack}. The approximation error $\epsilon$ decays based on the number of non-zero parametres, the width $W$, $w$ of the network. The decay rate is controlled by $\omega^*/d$ which makes sense as the smoothness $\omega^*$ increases the function is easier to approximate; inversely it is harder to approximate as the intrinsic dimension $d$ of the input space increases.\\
Similar observations can be made about the measure of $\nu$; whose decay rate only depends on the dimension $d$, not the smoothness.
\section{Convergence of Adversarial training}

Let $f^*$ be the robust ($p$,$\delta$) model \eqref{adversarial_training}, this is the true robust model in some arbitrary function space $\cl{C}$. We would like to study how the empirical estimate of $\widehat{f}_n\in \cl{C}_n$ w.r.t sample size $n$ converges to $f^*$. Corresponding to any $f\in \cl{C}$ we also have the best $\ell_p(\delta)$ function attack $\lambda^f(x,y)\in \pfc (\tilde{\cl{C}})$  which maximises the adversarial loss \eqref{best_function_attack}. For notation convenience now on wards denote $\lambda^f\coloneqq\lambda^f(x,y)$ for any model $f$; also $\Lambda\coloneqq \pfc (\tilde{\cl{C}})$ original function space of attacks. Adversarial training to find the robust $(p,\delta)$ model $f^*$ is the sequential mini-max problem:
\begin{align*}
    \lambda^f&=\arg\max_{\lambda\in\Lambda} \bb{E}\cl{L}(f,\lambda,\cl{Z}),\quad \forall f\in \cl{C},\\
    f^*&=\arg\min_{f\in\cl{C}}\bb{E}\cl{L}(f,\lambda^f,\cl{Z}).
\end{align*}
 In order to find the empirical estimator $\widehat{f}_n$ of the above sequential mini-max we can restrict the function spaces to $\cl{C}_n,\Lambda_n$ and optimize the empirical adversarial loss such that:
 \begin{align}
     \widehat{\lambda}^{f}_n\in \Lambda_n:\bb{E}_n\cl{L}(f,\widehat{\lambda}^{f}_n,\cl{Z}) &\geq\sup_{\lambda\in\Lambda_n}  \bb{E}_n\cl{L}(f,\lambda,\cl{Z})-\eta_n,\quad \forall f\in \cl{C}
     \label{def:A-estimatoreq1}\\
     \widehat{f}_n\in\cl{C}_n: \bb{E}_n\cl{L}(\widehat{f}_n,\widehat{\lambda}_n^{\widehat{f}_n},\cl{Z}) &\leq\inf_{f\in\cl{C}_n}\bb{E}_n\cl{L}(f,\widehat{\lambda}^{f}_n,\cl{Z})+\Bar\eta_n.
     \label{def:A-estimatoreq2}
 \end{align}
$\bb{E}_n$ is just the empirical expectation over the observed $n$-sample data $\{(x_i,y_i)\}_{i=1}^n$, $\widehat{\lambda}^{f}_n,\widehat{f}_n$ are empirical etimators of $\lambda^f,f^*$ respectively. Notice that $\widehat{\lambda}^{f}_n,\widehat{f}_n$ might not reach the supremum and infimum respectively w.r.t. the adversarial loss exactly but are very close to it, up to some small error values $\eta_n,\Bar\eta_n$. This is done, as $\cl{C}_n,\Lambda_n$ (Sieve-spaces \cite{Grenander1981AbstractInference}) might not be closed in general.

\subsubsection*{Zero-Sum game and Nash condition}
In the real-life implementation, the above mini-max optimization faces practical bottlenecks due to equation \eqref{def:A-estimatoreq1} as we need to find $\widehat{\lambda}^{f}_n$ for all possible $f\in \cl{C}$ then proceed to minimization problem \eqref{def:A-estimatoreq2}. Fortunately, the adversarial training framework is a special form of mini-max optimization.

Let $V=\bb{E}\cl{L}(f,\lambda,\cl{Z})$, then a robust model is trying to minimize $V$ and a $\ell_p(\delta)$ attack is trying to minimize $-V$. This sort of mini-max optimization is known as a \textit{Zero-Sum} game \cite{TheoryJSTOR}. A classical result in game theory is that Nash equilibrium is attained at the solution of a zero-sum mini-max optimization. Hence, an equivalent optimization scheme to \eqref{def:A-estimatoreq1},\eqref{def:A-estimatoreq2} based on empirical Nash condition will be as follows:
\begin{align}
    \widehat{\lambda}_n\in \Lambda_n:\bb{E}_n\cl{L}(\widehat{f}_n,\widehat{\lambda}_n,\cl{Z}) &\geq\sup_{\lambda\in\Lambda_n}  \bb{E}_n\cl{L}(\widehat{f}_n,\lambda,\cl{Z})-\eta_n
     \label{def:Nasheq1}\\
     \widehat{f}_n\in\cl{C}_n: \bb{E}_n\cl{L}(\widehat{f}_n,\widehat{\lambda}_n,\cl{Z}) &\leq\inf_{f\in\cl{C}_n}\bb{E}_n\cl{L}(f,\widehat{\lambda}_n,\cl{Z})+\widetilde{\eta}_n
     \label{def:Nasheq2}
\end{align}
The above scheme is like a solution to a two-player game where the best robust model is minimization of the loss w.r.t. the best attack and the best attack itself is maximization of the loss w.r.t. the best robust model. The Nash equilibrium formulation is useful when we do adversarial training in practice; instead of exploring the whole parameter/function space we just iteratively find the best attack given a model and then robustly train the model given the best attack and continue the process until we reach equilibrium.

\subsubsection*{Convergence rate}
We will state convergence results on the expected difference in loss based on the empirical estimate $\widehat{f}_n$ and $f^*$; also $\widehat{\lambda}^f$ and $\widehat{\lambda}_n^f$ for any $f$ with respect to the sample size $n$. Also, we assume that the original parameter space for the model is a product of H\"older space i.e. $f\in \cl{C}=\prod_{i}^K\cl{H}(\bar{\omega}_i,\cl{X})$ with smoothness $\bar{\omega}^*=\min\{\bar{\omega}_i\}_{i=1}^n$.

\begin{theorem}[Adversarial training rates]
\label{thm:convergence}
Let $\cl{C}=\prod_{i}^K\cl{H}(\bar{\omega}_i,\cl{X})$  and $\Lambda=\pfc(\phol)$, and we assume that $\Lambda, \cl{C}$ are compact under $\|\cdot\|_\infty$. For all $\lambda,\lambda'\in\Lambda,\ f,f'\in\cl{C}$, assume the following conditions hold:
\begin{itemize}
    \item A0: Assume that $f^*\in \cl{C}$ with support  $ \Bar {\mathcal{X}}_s\subset[0,1]^{ D}$ with $\operatorname{dim}_M \Bar {\mathcal{X}}_s = \Bar d^* \leq  D $ (see Def. \ref{def:minkowski} and \ref{def:holder}),
    and that $\{\lambda_*^f : f\in\cl{C}\}\subset\Lambda$  on support ${\mathcal{X}_s} \subset[0,1]^{D }$ with $\operatorname{dim}_M {\mathcal{X}_s}  = d^* \leq D $. 
    \item A1: The loss is Lipschitz continuous in $f$ and $\lambda$,
    $$\cl{L}(f, \lambda, \cl{Z})-\cl{L}(f', \lambda', \cl{Z}) \prec \|f-f' \|_{\Bar {\mathcal{X}} } + \|\lambda-\lambda'\|_{{\mathcal{X}} },$$
    where $\norm{}_{\cl{X}},\norm{}_{\Tilde{\cl{X}}}$ are the supremum norms over the supports $\cl{X},\cl{\Tilde{\cl{X}}}$. 
    \item A2: The variance of the loss difference is bounded by its expectation, specifically in terms of $f$ as follows: $$\mathbb{V}[\cl{L}(f,\lambda^{f}, \cl{Z})-\cl{L}(f^*,\lambda^{f^*}, \cl{Z})] \prec \mathbb{E}[\cl{L}(f,\lambda^{f}, \cl{Z})-\cl{L}(f^*,\lambda^{f^*},\cl{Z})] \prec \|f-f^*\|^{2}_{\widetilde{{\mathcal{X}}}} $$
    for any $\widetilde{{\mathcal{X}}}\subset\Bar {\mathcal{X}}_s$.
    \item A3: The variance of the loss difference is bounded by its expectation, specifically in terms of $\lambda$ as follows:
    $$\mathbb{V}[\cl{L}(f,\lambda^f, \cl{Z})-\cl{L}(f,\lambda, \cl{Z})] \prec \mathbb{E}[\cl{L}(f,\lambda^f, \cl{Z})-\cl{L}(f,\lambda, \cl{Z})]\prec \|\lambda^f-\lambda\|^{2}_{\widetilde{{\mathcal{X}}}} + \mathbb{P}^{2\omega^*}(x\not\in \widetilde{\mathcal{X}})$$ for any  $\widetilde{{\mathcal{X}}}\subset {\mathcal{X}}_s$.
\end{itemize}
Pick any two values $\Bar r>\underline r\geq \left(\frac{d^*}{\omega^*}\vee \frac{\Bar d^*}{\Bar \omega^*}\right)$. Consider estimators $\widehat{f}_n,\widehat{\lambda}_n^f$ in equations \eqref{def:A-estimatoreq1},\eqref{def:A-estimatoreq2} with $\eta_n, \Bar\eta_n = o_\mathbb{P}(n^{-2/(2+ \Bar r)})$ where $\Lambda_n=\mathcal{F}_{\sigma}(L, W_n, w_n, \kappa_n)$ and $\cl{C}_n=\mathcal{F}_{\sigma}(\Bar L, \Bar W_n, \Bar w_n, \Bar \kappa_n)$ implement neural networks (cf. Definition \ref{def:neuralnets}) satisfying $W_n, \Bar W_n, w_n, \Bar w_n \asymp n^{\underline r/(\underline r+2)}$ and $\kappa_n, \Bar \kappa_n \asymp n^{c}$ for any large enough choices of $L,\Bar L, c>0$. Then:
\[\mathbb{E}[\cl{L}(\widehat{f}_n,\lambda^{\widehat{f}_n}, \cl{Z})-\cl{L}(f^*,\lambda^{f^*}, \tilde{\cl{Y})}]=o_{\mathbb{P}}(n^{-2/(2+\Bar r)}),\]
\[\sup_{f\in\cl{C}_n} \mathbb{E}[\cl{L}(f,\lambda^f, \cl{Z})-\cl{L}(f,\widehat\lambda_n^f, \cl{Z})]=o_{\mathbb{P}}(n^{-2/(2+\Bar r)}).\]
Hence, $d(\widehat{f}_n, f^*)=o_\mathbb{P}(1)$ for any (pseudo-)norm $d(\cdot,\cdot)$ under which $\mathbb{E}[l(f, \lambda^f,\tilde{\cl{Y})}]$ is compact and continuous. Further, if $d(f, f^*)^{1/q}\prec \mathbb{E}[\cl{L}(f, \cl{Z})-\cl{L}(f^*,\cl{Z})]$ for $q>0$, we get:\looseness=-1
$$d(\widehat{f}_n, f^*)=o_\mathbb{P}(n^{-2q/(2+\Bar r)}).$$
\end{theorem}

We have the following remarks regarding the assumptions in the above theorem.
\begin{remark}

$A_0$ is a trivial assumption ensuring that our robust model $f^*$ belongs to a product space of H\"older continuous functions. This is a reasonable assumption for regression/classification problems in the deep learning realm. Also, the best attack $\lambda^f$ for any $f\in \cl{C}$ belongs to the product space of piece-wise H\"older functions; this is supported by our discussions in section \ref{subsec:motivation_piece}.
$A1$ enforces Lipschitz continuity for our adversarial loss $\cl{L}(f,\lambda,\cl{Z})$ w.r.t. $\lambda$ and $f$. 
$A2, A3$ suggests that the variance and mean of difference in adversarial loss can be controlled by the difference of the underlying parameters. In $A3$, we have the extra term $\mathbb{P}^{2\omega^*}(x\not\in \widetilde{\mathcal{X}})$. This will be used to account for the small measure set where our neural network can't approximate the attack well (Refer Theorem \ref{thm:NNapprox}). That is, we will let $\widetilde{\mathcal{X}}$ be the measurable set where approximation fails and $\omega^*$ is the smoothness factor for piece-wise H\"older function space.
\end{remark}

\section{Numerical Experiments}
As discussed earlier the sequential mini-max optimization of our adversarial loss \eqref{def:A-estimatoreq1}\eqref{def:A-estimatoreq2} is a zero-sum game and must reach the Nash-equilibrium at the optimum \eqref{def:Nasheq1}\eqref{def:Nasheq2}. Empirically we can reach such an equilibrium by emulating the adversarial \textit{Zero-sum} game between the estimators of the robust model $f$ \eqref{adversarial_training} and corresponding best attack $\lambda^f$ \eqref{best_function_attack}.

Motivated by this, given a sample of $N$ data points we train two neural networks $f,\lambda^f$ with parameters $\theta_f,\theta_{\lambda}$; competing against each other until convergence or Nash-equilibrium is reached (Algorithm: \ref{alg:adv_training}). $T$ is the number of training epochs; $P$ represents the number of iterations the model $f$ is trained in each epoch; correspondingly $H$ is the number of iterations $\lambda^f$ is trained in one epoch; $\tilde{\gamma},\Bar\gamma$ are the learning rates for $f,\lambda^f$. At each training epoch, $f$ is optimized in order to minimize the adversarial loss given the current state of $\lambda^f$; and $\lambda^f$ maximizes the adversarial loss (minimizes negative adversarial loss) given the current state of $f$. This is in a similar flavor to the training of general adversarial networks (GANs) \cite{Goodfellow2014GenerativeNetworks}, where the discriminator plays an analogous role to our robust model $f$ and the generator is analogous to the attack model $\lambda^f$ in our scenario.

\begin{algorithm}
\small
\caption{Adversarial Training}\label{alg:adv_training}
\begin{algorithmic}
\Initialize{$t\gets0,h\gets0,p\gets0$.
Each entry of $\theta_f,\theta_{\lambda}\sim^{iid} \mathcal{U}(-\sqrt{\frac{1}{\texttt{in\_features}}}, \sqrt{\frac{1}{\texttt{in\_features}}})$, where \texttt{in\_features} means the width of the previous layer in the network.} 
\While{$t\leq T$}
\While{$p\leq P$}
\State $\theta_{f}\gets \theta_{f}-\tilde{\gamma} \nabla_{\theta_f} \sum_{i=1}^n\cl{L}(f,\lambda^f,(x_i,y_i))$ \Comment{Can be any gradient-based optimizer, e.g., SGD, ADAM, etc }
\State $p\gets p+1$
\EndWhile
\While{$h \leq H$}
\State $\theta_{\lambda}\gets \theta_{\lambda}+\bar{\gamma} \nabla_{\theta_{\lambda}} \sum_{i=1}^n\cl{L}(f,\lambda^f,(x_i,y_i))$
\State $h\gets h+1$
\EndWhile
\State $p\gets 0, h\gets 0,t\gets t+1$
\EndWhile
\end{algorithmic}
\end{algorithm}

We apply the adversarial framework in both classification and regression settings. For the classification setting, we use the two-dimensional geometric datasets (see; Appendix \ref{A:exp_details} and Figure \ref{fig:dataset_classification}); we visually evaluate the consistency of the learned attack network with respect to the theoretical best attack \eqref{gradient_flow_attack} based on projected gradient flow. We also use the number of misclassifications and adversarial loss to evaluate the robustness, and attack strength of $f,\lambda^f$ respectively. For the regression setting, we use the (\texttt{Boston}) \cite{BostonHousing}, (\texttt{Diabetes}) \cite{Diabetes} datasets; we evaluate adversarial loss on the test dataset as a performance metric.

\subsection*{Classification}
We implement algorithm \ref{alg:adv_training} with $T=100, H=1, P=1$ and $\tilde{\gamma}=10^{-3},\Bar\gamma=2\times 10^{-4}$ for all the classification datasets. We use $\delta=0.2$ (attack strength) and $\ell_p:p=2,\infty$ for our experiment. 
When $\delta$ is big, the so-called \textit{robust overfitting} might occurs, and optimization results may not be satisfactory, and more discussion about it can be found in Section \ref{sec:dis}.
For all feasible values of $\delta$, the analysis remains the same.
Furthermore, the $\ell_p$ constraint on the attack network $\lambda^f$ is implemented in the forward pass through learned scaling (between 0 and 1) and normalization. Architecture and implementation details can be found in Appendix \ref{A:exp_details}. We use the multinomial negative loglikelihood or the cross-entropy loss as our loss function $\cl{L}(f(x),y)=\sum_{j=0}^{C-1}\bb{1}_{y=j}\log{\rm Softmax}(f(x))_j$; the corresponding adversarial loss is then $\cl{L}(f,\lambda^f,(x,y))=\cl{L}(f(x+\lambda^f(x,y)),y)=\sum_{j=0}^{C-1}\bb{1}_{y=j}\log{\rm Softmax}(f(x+\lambda^f(x,y)))_j$. Here, $C$ is the number of classes.
\begin{sidewaysfigure}
     \centering
     \begin{subfigure}[b]{0.2\textwidth}
         \centering
         \includegraphics[trim={0 3cm 2cm 3cm},clip,height=12cm]{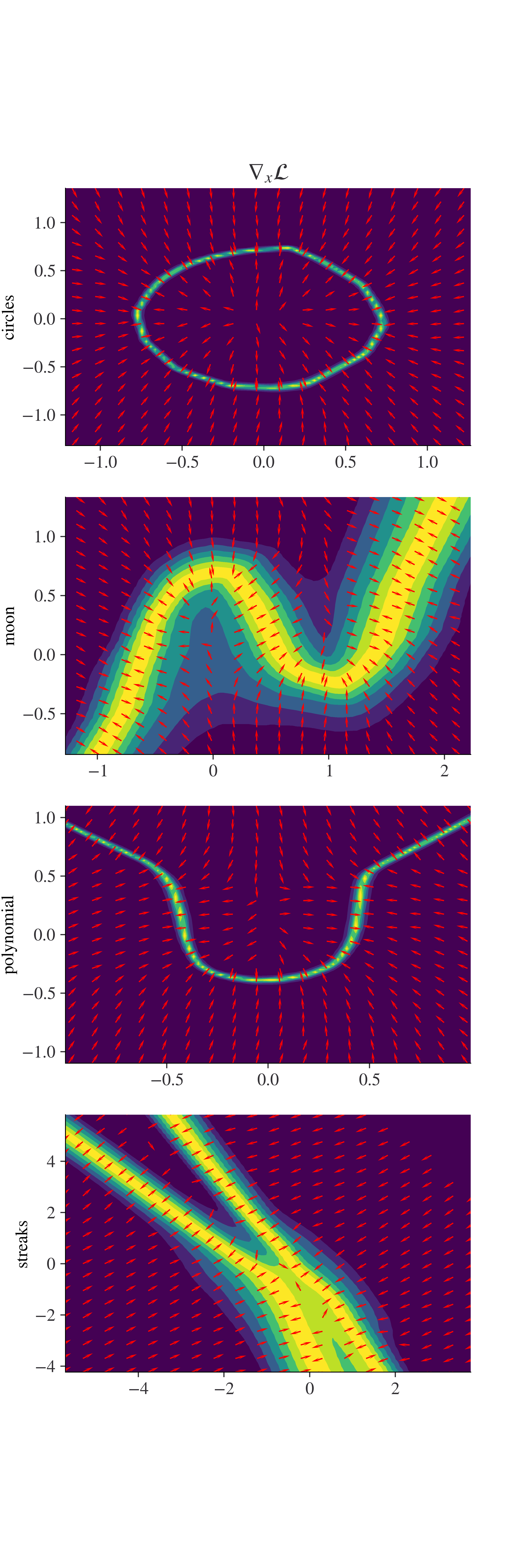}
         \caption{Clean Training}
         \label{fig:clean_training}
     \end{subfigure}
     \hfill
     \begin{subfigure}[b]{0.35\textwidth}
         \centering
         \includegraphics[trim={1.5cm 3cm 2cm 3cm},clip,width=7cm,height=12cm]{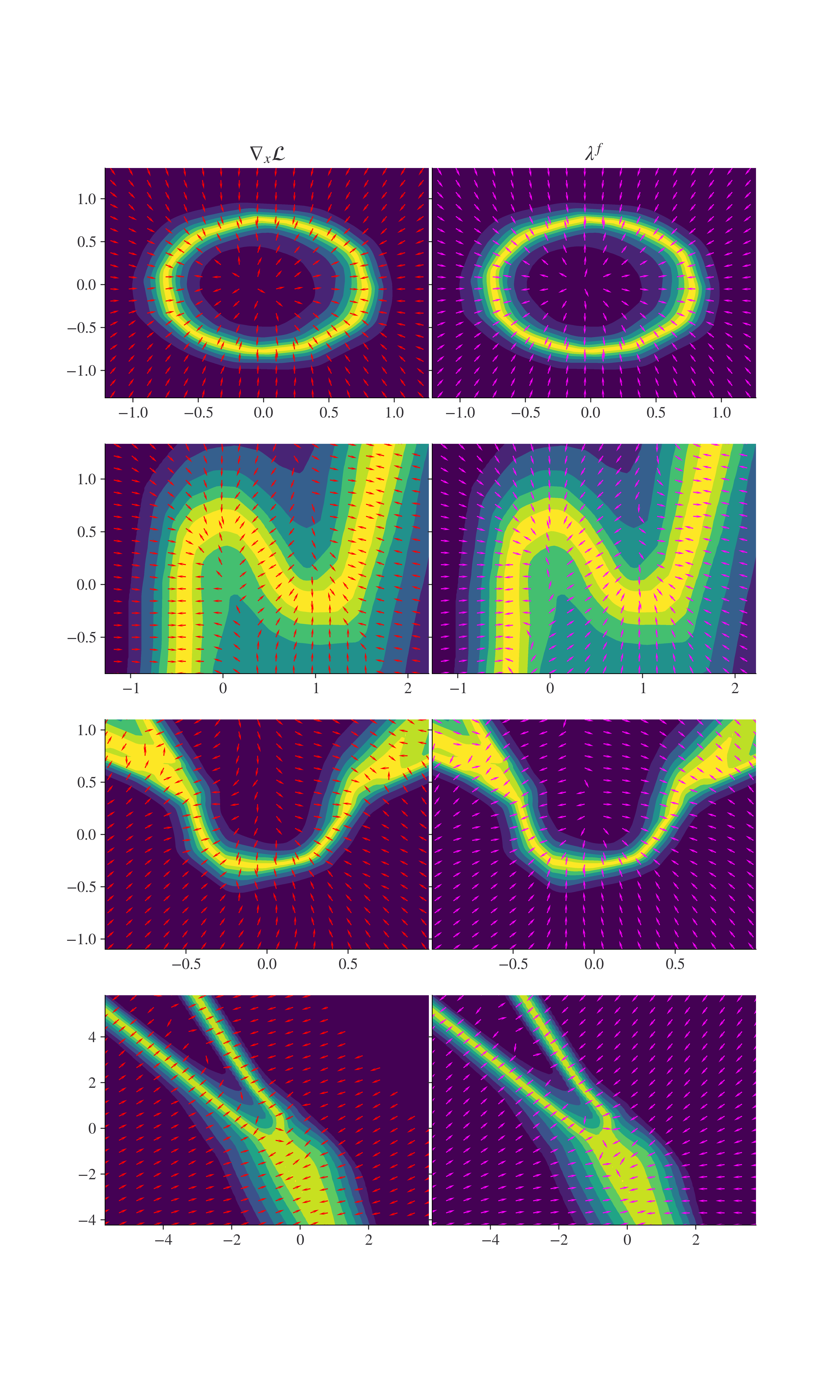}
         \caption{$\ell_{2}(0.2)$}
         \label{fig:l2attack}
     \end{subfigure}
     \hfill
     \begin{subfigure}[b]{0.35\textwidth}
         \centering
         \includegraphics[trim={1.5cm 3cm 1cm 3cm},clip,width=7cm,height=12cm]{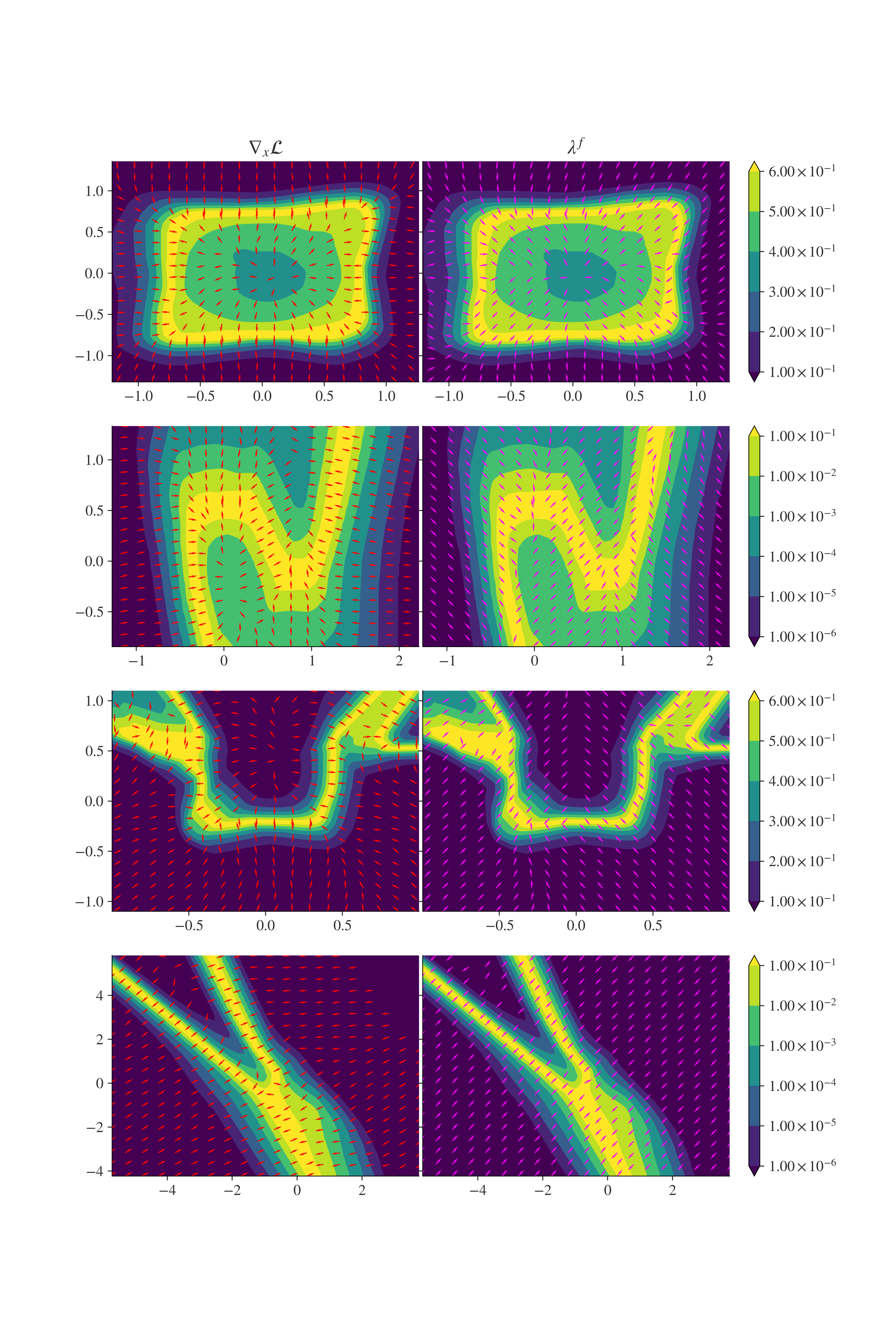}
         \caption{$\ell_{\infty}(0.2)$}
         \label{fig:linfinityattack}
     \end{subfigure}
        \caption{\textbf{(a)} Clean trained model loss, the red arrows represent the gradient direction of the loss at each point. \textbf{(b)} Adversarially trained model based on algorithm \ref{alg:adv_training} with $\norm{\lambda^f}_{2}\leq 0.2$. \textbf{(c)} Adversarially trained model based on algorithm \ref{alg:adv_training} with $\norm{\lambda^f}_{\infty}\leq 0.2$. For figures (b) and (c); [Left]: Red arrows represent the gradient of the adversarially trained model loss. [Right]: Purple arrows represent the attack generated by $\lambda^f$. The contour colors represent the value of the corresponding loss at each point.
        }
        \label{fig:gradientvsNN}
\end{sidewaysfigure}
\subsubsection*{Gradient flow and Neural Network attack direction}
Figure \ref{fig:clean_training}, shows the loss surface and gradient direction $\nabla_x \cl{L}(f_{clean}(x),y))/\|\nabla_x \cl{L}(f_{clean}(x),y))\|$ corresponding to $x$'s on a grid, where the response label $y$ is imputed by
model $f_{clean}$, which achieves almost 100\% accuracy and is viewed as an oracle. 
Here $f_{clean}$ is the clean model trained to minimize the standard loss ($f_{clean}=\arg\min_{f\in \cl{C}_n} \bb{E}_n\cl{L}(f(x),y)$) without any adversarial perturbation. The higher loss regions correspond to the \textit{estimated decision boundary} (EDB) of the dataset; this is consistent as the loss is locally maximized when any pair of classes have approximately equal probability.

Figure \ref{fig:l2attack} shows the loss surface, gradient direction, and learned attack direction of $\lambda^f$ for the adversarially trained $\ell_2(0.2)$ robust model $f$ based on Algorithm \ref{alg:adv_training}. As discussed in Section \ref{subsubsec:pgf}, \eqref{gradient_flow_attack} is the best theoretical attack. Ignoring the deflection term for saddle points, the underlying ODE is dictated by a projected gradient flow (\ref{progected_gradient_flow}, \ref{eqn:lp_notinf_ODE}).
For $p=2$, the projected gradient flow projects the original gradient flow onto local spheres. Due to the smoothness of spheres in every direction, we expect a flow visually similar to the original gradient flow.
This is what we see in figure \ref{fig:l2attack}, where $\nabla_x \cl{L}(f(x),y)$ (left) and $\lambda^f$ (right) are essentially the same. This also corroborates the approximation capability of neural network for piece-wise H\"older function space (Theorem \ref{thm:NNapprox}) where the theoretical best attack resides. 

Figure \ref{fig:linfinityattack} is the $\ell_{\infty}(0.2)$ case of the prior discussion. Unlike the $\ell_2$ case, for $\ell_{\infty}$, the projected gradient flow projects the gradient trajectory onto local cubes, which are not smooth in all directions. The gradient flow dynamic continues after reaching the boundary and gets projected onto the sides/diagonals of a cube \eqref{eqn:l_inf_ODE}. This is consistent with Figure \ref{fig:linfinityattack} where $\nabla_x \cl{L}(f(x),y)$ (left) is more smooth, but $\lambda^f$ (right) is restricted only to diagonal directions.

Subsequently, comparing the adversarially trained EDB with clean trained EDB we see that the EDB is wider after adversarial training which essentially induces the \textit{robustness} in the model. Also, due to the prior projection arguments $\ell_2$ adversarial training preserves the geometry of the EDB while $\ell_{\infty}$ projects the original EDB onto local cubes, making it more blocky; this phenomenon is most prominent for the circles dataset where EDB transforms from a circle ($\ell_2$) to a square ($\ell_{\infty}$) as expected.
\begin{figure}
    \centering
    \includegraphics[scale=0.6]{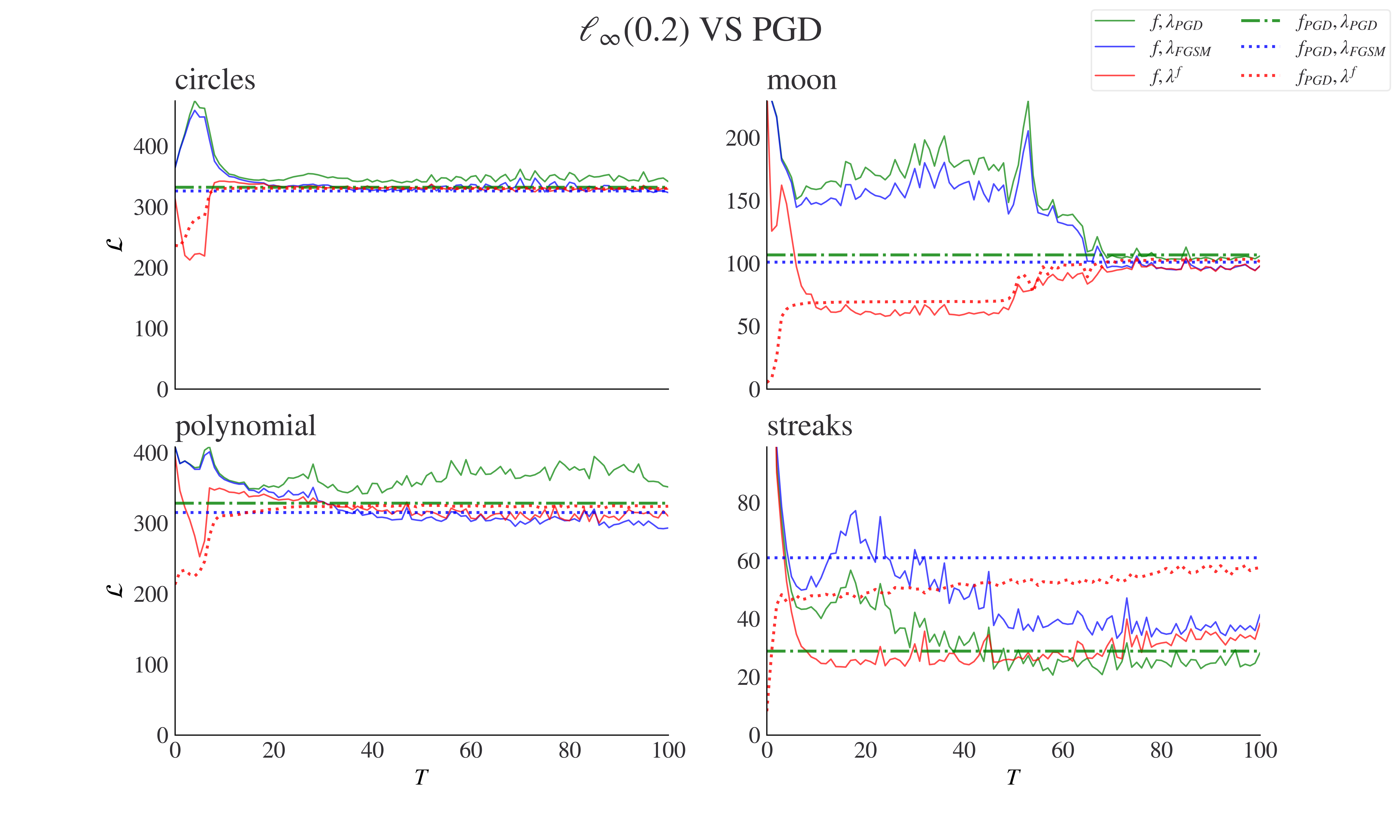}
    \caption{\textbf{(Classification setting)} Test Loss vs training epochs for adversarial training based on Algorithm \ref{alg:adv_training}. Where $f$ is our adversarially trained classifier and $f_{PGD}$ PGD trained robust classifier. For the attacks: $\lambda^f$ represents adversarially trained neural network attack based on $f$ \textit{(semi-white box)}; $\lambda_{PGD}$ represents \textit{(white-box)} PGD attack; $\lambda_{FGSM}$ represents \textit{(white-box)} FGSM attack.\\ In the legends the combination $(f,\lambda_{PGD})$ represents the empirical loss $\sum_{i=1}^{n_{test}}\cl{L}(f,\lambda_{PGD},(x_i,y_i))$ when the PGD attack $\lambda_{PGD}$ is performed on the adversarially trained defense $f$, similarly $(f_{PGD},\lambda_{PGD})$ represents the empirical loss $\sum_{i=1}^{n_{test}}(f_{PGD},\lambda_{PGD},(x_i,y_i))$ when $\lambda_{PGD}$ attack is performed on $f_{PGD}$; and so on.}
    \label{fig:loss_plot}
\end{figure}
\begin{figure}
    \centering
    \includegraphics[scale=0.6]{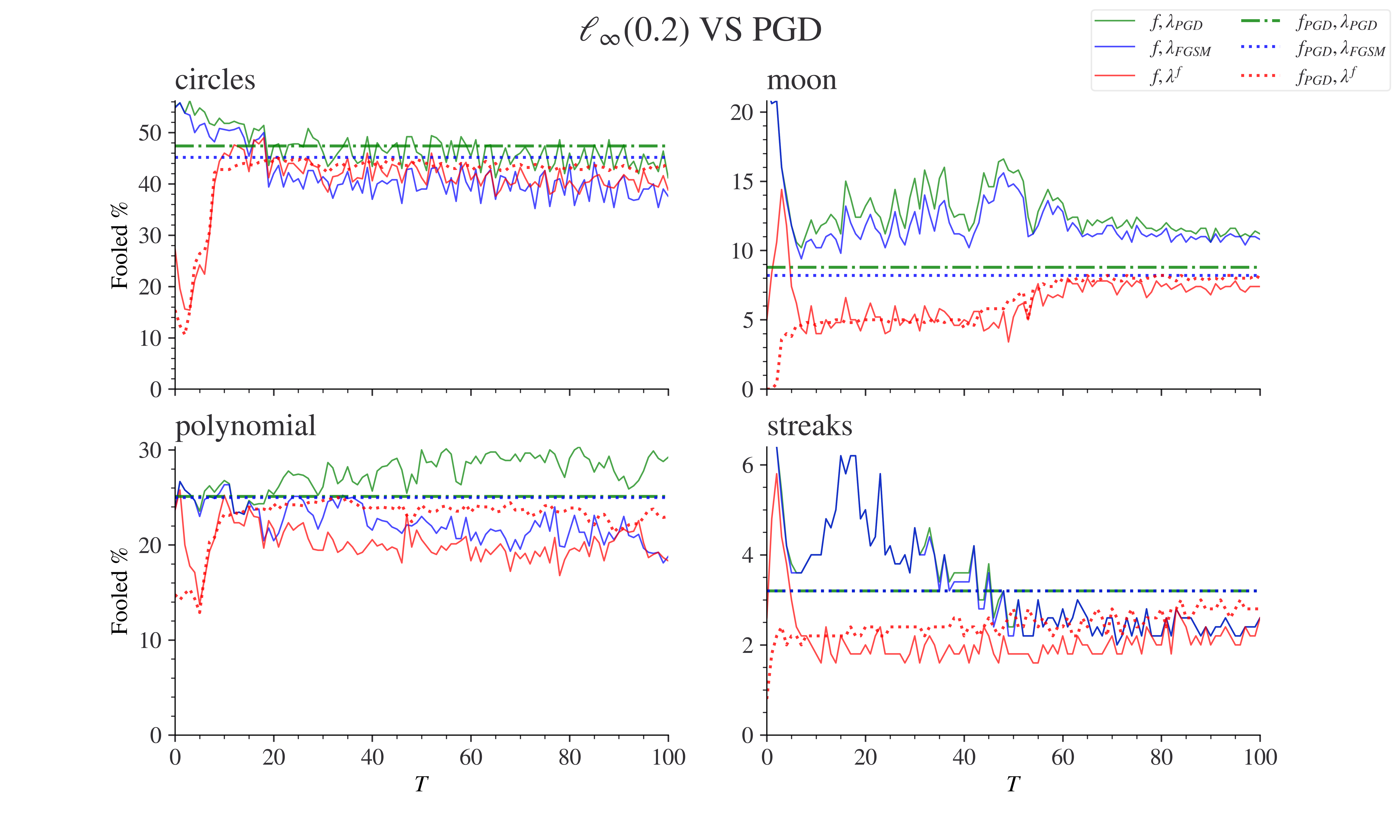}
    \caption{Percentage of mis-classified test labels vs training epochs for adversarial training based on Algorithm \ref{alg:adv_training}. 
    The notation and legend in the plots follow Figure \ref{fig:loss_plot}.
    }
    \label{fig:performance_plot}
\end{figure}

\subsubsection*{Comparison with PGD attacks and PGD trained robust classifer}
In section \ref{subsec:discrete_best_atk} we discussed how PGD attack \eqref{pgd_attack} \cite{Madry2017TowardsAttacks} is just a discretization of the theoretical best possible attack \eqref{gradient_flow_attack}. With small enough step-size $\gamma$ we can treat PGD attack as a baseline (i.e., discrete approximation of best attack), for notational convenience we will denote $\lambda_{PGD}$ as the PGD attack.  Similarly, $\lambda_{FGSM}$ denotes the FGSM attack mentioned in \cite{Goodfellow2014ExplainingExamples} which is just PGD attack with $\gamma=\delta$ and a single iteration step. 
We train $f_{PGD}$ which aims to minimize the PGD-based adversarial loss $\bb{E}_n\cl{L}(f(x+\lambda_{PGD}),y)$, this is the traditional way of adversarial training mentioned in \cite{Madry2017TowardsAttacks}. In our framework, the attack generation itself is approximated by a neural network $\lambda_f$ rather than a gradient-based method, and the adversarial training becomes a mathematical game between $f$ and $\lambda_f$. The significance of $f_{PGD}$ is that with small step-size $\gamma$ for PGD attack generation, it will act as a baseline for our version of adversarial training.

Figures \ref{fig:loss_plot} and \ref{fig:performance_plot} show the evaluated adversarial loss and misclassification rate for all attack-defense combinations on the \textit{test data} as we progress through adversarial training. For all the datasets we see that the adversarial loss for our robust network $f$ with attacks $\lambda^f,\lambda_{PGD},\lambda_{FGSM}$ converge to their respective $f_{PGD}$ counter-parts. The loss that our framework deals with is corresponding to the $(f,\lambda^f)$ pair (solid red line) and the baseline would be $(f_{PGD},\lambda_{PGD})$ (dashed green line). For the \texttt{streaks} dataset there seems to be a gap between PGD defense $f_{PGD}$ against FGSM attack $\lambda_{FGSM}$ and PGD attack $\lambda_{PGD}$; we see that our adversarially trained attack $\lambda^f$ tries to match $\lambda_{FGSM}$ attack on $f_{PGD}$. On the other hand our defense $f$ tries to match the loss levels of $f_{PGD}$ defense against its own $\lambda_{PGD}$ attack. Comparing the solid curves (defense model $f$) with non-solid curves (defense model $f_{PGD}$), our defense model has a lower loss under the same attack, indicating its higher adversarial robustness.

For the misclassification rate, we see that,  for \texttt{circles,streaks} datasets, $f$ shows stronger defense for all types of attacks compared to $f_{PGD}$ but the attack strength of $\lambda^f$ is weaker compared to $\lambda_{PGD}$. For \texttt{moon} the opposite happens where the $\lambda^f$ shows misclassifications comparable to $\lambda_{PGD}$ but the defense $f$ is a bit weaker than $f_{PGD}$ but still competitive. For \texttt{polynomials}, our defense $f$ performs better against $\lambda_{FGSM}$ compared to $f_{PGD}$, but $f_{PGD}$ is a bit better to defend against $\lambda_{PGD}$. This is not that shocking as $f_{PGD}$ should perform well against $\lambda_{PGD}$ as it is adversarially trained over it.
To comment on the differences between PGD and our method, we note that our attack/defense model is trained solely based on loss values, while the implemented PGD algorithm in our simulation combines the loss value and misclassification. \footnote{PGD, for classification, follows the standard PGD iterations on each sample until it is misclassified. If a sample is misclassified, we do not implement further PGD iterations. Early stopping until misclassification maximizes the misclassification rate. \tiny{\url{https://github.com/rhaldarpurdue/Adversarial-NN-simulations/blob/main/utils.py}}}
Furthermore, the PGD attack is a white box attack, while our $\lambda_f$ is a \textit{semi-white box} attack.


\subsubsection*{White box vs Semi-white box attacks:}
Note that technically the neural network-based attack $\lambda^f$ is at a disadvantage compared to the gradient-based attacks $\lambda_{PGD},\lambda_{FGSM}$. $\lambda^f$ learns the functional form of the attack and doesn't have access to the original model while generating new attacks on the test dataset, this is known as a \textit{semi-white} box setting, i.e., once the model has been trained no extra gradient computations are done.

In contrast, $\lambda_{PGD},\lambda_{FGSM}$ need to be computed from scratch for each new sample in the test dataset and also need to have access to the original classifier models $f,f_{PGD}$. These attacks hence fall into the \textit{white-box} setting.

The competitive performance based on test loss and misclassification rate against various defense/attack combinations compared to the gradient-based method is consistent with the theory. In real life, one saves computation costs while using our framework, as $\lambda^f$ approximates \eqref{gradient_flow_attack} using neural networks. Also, the defense $f$ learned as a consequence provides competitive performance or sometimes better robustness based on these experiments.

\section*{Regression}
For the \texttt{Boston,Diabetes} datasets, we implement Algorithm \ref{alg:adv_training} with $T=400,H=1,P=1, \tilde{\gamma}=10^{-3},\Bar\gamma=2\times10^{-4}$. We use $\ell_2$ attacks to compare the adversarial loss between $f_{PGD}$ and $f$ under $\lambda_{FGSM},\lambda_{PGD},\lambda^f$ attacks on the test data. As $\ell_2$ attacks preserve geometry, that perturbation strength in any direction is the same; we deemed it to be an appropriate choice for the regression setting. For the loss, we use the Gaussian negative log-likelihood or the mean-square error loss, $\cl{L}(f(x),y)=\norm{f(x)-y}_2^2$. The corresponding adversarial loss being $\cl{L}(f,\lambda^f,(x,y))=\cl{L}(f(x+\lambda^f(x,y)),y)=\norm{f(x+\lambda^f(x))-y}_2^2$ 

Figure \ref{fig:loss_plot_reg_0.2} shows the convergence of the adversarial loss for  our robust defense $f$ based on algorithm \ref{alg:adv_training}, comparing against the baseline $f_{PGD}$. In Appendix \ref{A:exp_details}, it can be seen that for higher $\delta$, $\lambda^f$ is stronger than $\lambda_{PGD}$.
\begin{figure}
    \centering
    \includegraphics[scale=0.6]{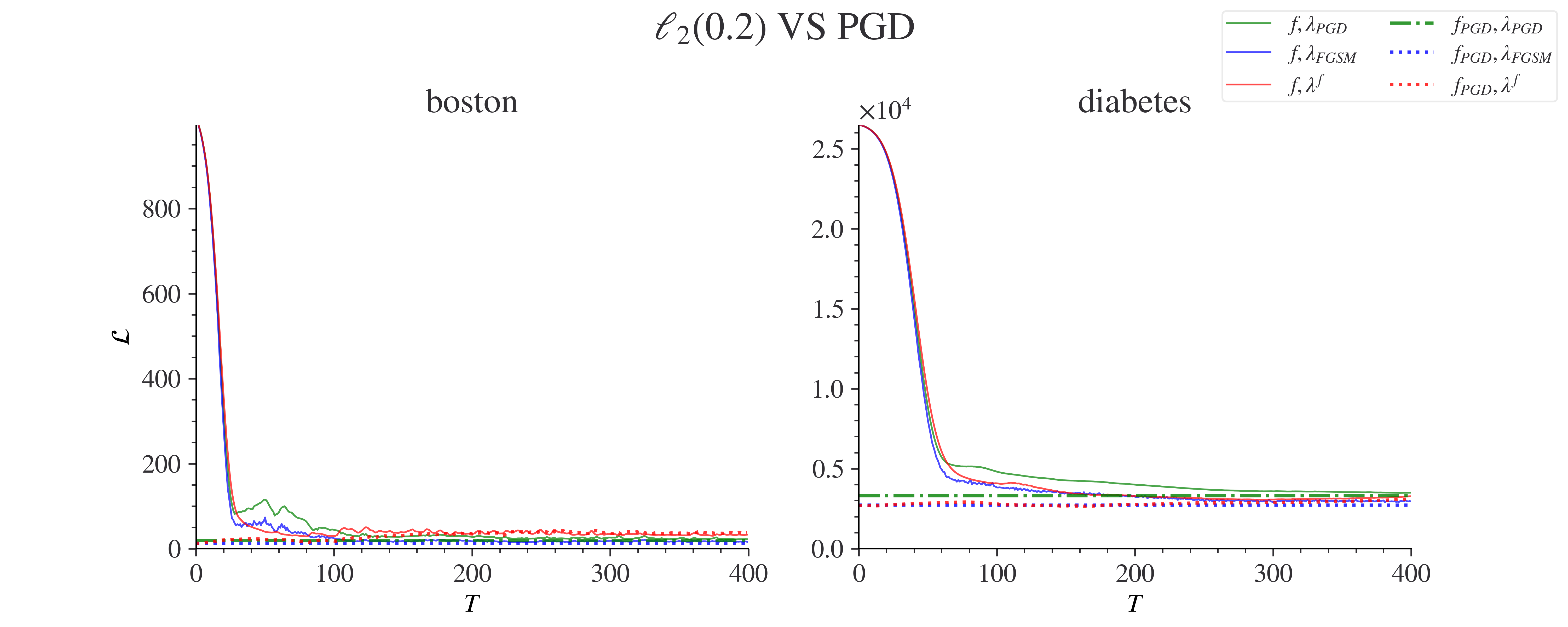}
    \caption{\textbf{(Regression setting)} Test Loss vs training epochs for adversarial training based on Algorithm \ref{alg:adv_training}. Where $f$ is our adversarially trained classifier and $f_{PGD}$ PGD trained robust classifier. For the attacks: $\lambda^f$ represents adversarially trained neural network attack based on $f$ \textit{(semi-white box)}; $\lambda_{PGD}$ represents \textit{(white-box)} PGD attack; $\lambda_{FGSM}$ represents \textit{(white-box)} FGSM attack.\\ In the legends the combination $(f,\lambda_{PGD})$ represents the empirical loss $\sum_{i=1}^{n_{test}}\cl{L}(f,\lambda_{PGD},(x_i,y_i))$ when the PGD attack $\lambda_{PGD}$ is performed on the adversarially trained defense $f$, similarly $(f_{PGD},\lambda_{PGD})$ represents the empirical loss $\sum_{i=1}^{n_{test}}(f_{PGD},\lambda_{PGD},(x_i,y_i))$ when $\lambda_{PGD}$ is performed on $f_{PGD}$; and so on.}
    \label{fig:loss_plot_reg_0.2}
\end{figure}

\section{Discussion and conclusion}\label{sec:dis}
\subsubsection*{Robust Overfitting}
When $\delta$ is comparatively large, adversarial trained robust model minimizing the objective \eqref{adversarial_training} tends to prioritize minimizing the loss towards generated adversarial examples and forget the original geometry of data. This could lead to distortion of the loss surface compared to the clean model and poor generalization performance of the robust model. This issue is known as \textit{Robust Overfitting} in the literature. \cite{Xiao2022StabilityTraining} provides theoretical rates for reduction of generalisation error w.r.t. $\delta$ attack strength. \cite{Yu2022UnderstandingBeyond} suggests a loss-constraining method to prevent data overfitting.

From our experiments, we see that robust overfitting arises as a distortion of the loss surface. Figure \ref{fig:alpha_compare} (left) shows the loss surface of adversarially trained robust $\ell_{\infty}(0.2)$ robust model. The robust classifier tends to defend against the perturbed data samples and as a result, forgets about the original geometry of the data points which results in distortion of the EDB. There is a trade-off between robustness and standard accuracy \cite{Tsipras2018RobustnessAccuracy}. As $\delta$ increases the distortion will get worse, as the perturbed data is farther from the original data. We can resolve this issue by retaining some original data information while conducting adversarial training; one can modify the standard adversarial loss as follows:
\begin{equation}
    \cl{L}(f,\lambda^f,\cl{Z})=\bb{E}\cl{L}(f(x+\lambda^f(x,y)),y)\to \bb{E}(1-\alpha)\cl{L}(f(x+\lambda^f(x,y)),y)+\alpha\cl{L}(f(x),y)
    \label{alpha_tradeoff}
\end{equation}
The above equation doesn't change any of our theory, all the convergence results and arguments still hold as they are true for a general $\cl{L}(f,\lambda^f,\cl{Z})$. Also, the non-adversarial counterpart of the loss function is still our original vanilla loss used for clean training. This is a canonical modification that assigns $\alpha$ weightage to the original data. For the modified loss, the performance plots analogous to Figures \ref{fig:performance_plot} and \ref{fig:loss_plot}  are almost identical, and we omit them for brevity. The advantage of this approach is that we eliminate the possible distortions in the loss surface (Figure \ref{fig:alpha_compare} Right). 
This gets rid of the robust overfitting issue and even for high $\delta$ the robust model can't stray away completely from the original data (Figure \ref{fig:alpha_for_various_delta}).

\cite{Gowal2021ImprovingData} proposes a similar weighted loss but instead of using original data, they use a generative model to produce synthetic data which are perturbed. The idea being that synthetic data are purer in some sense and retain important features even after perturbation.

Another popular method to deal with robust overfitting is \textit{trade adversarial robustness off against accuracy} (TRADES) proposed by \cite{Zhang2019TheoreticallyAccuracy}. TRADES splits the adversarial loss in a similar fashion as \eqref{alpha_tradeoff}; but uses the original classifier scores instead of the true labels for the adversarial part. In our framework this will correspond to the following adversarial loss: \begin{equation}
    \cl{L}(f,\lambda^f,\cl{Z})= \bb{E}(1-\alpha)\cl{L}(f(x+\lambda^f(x,y)),f(x))+\alpha\cl{L}(f(x),y).
\end{equation}
One can extend these methods to our framework by changing the adversarial loss $\cl{L}(f,\lambda^f,\cl{Z})$ suitably.
\begin{figure}
    \centering
    \includegraphics[scale=0.5]{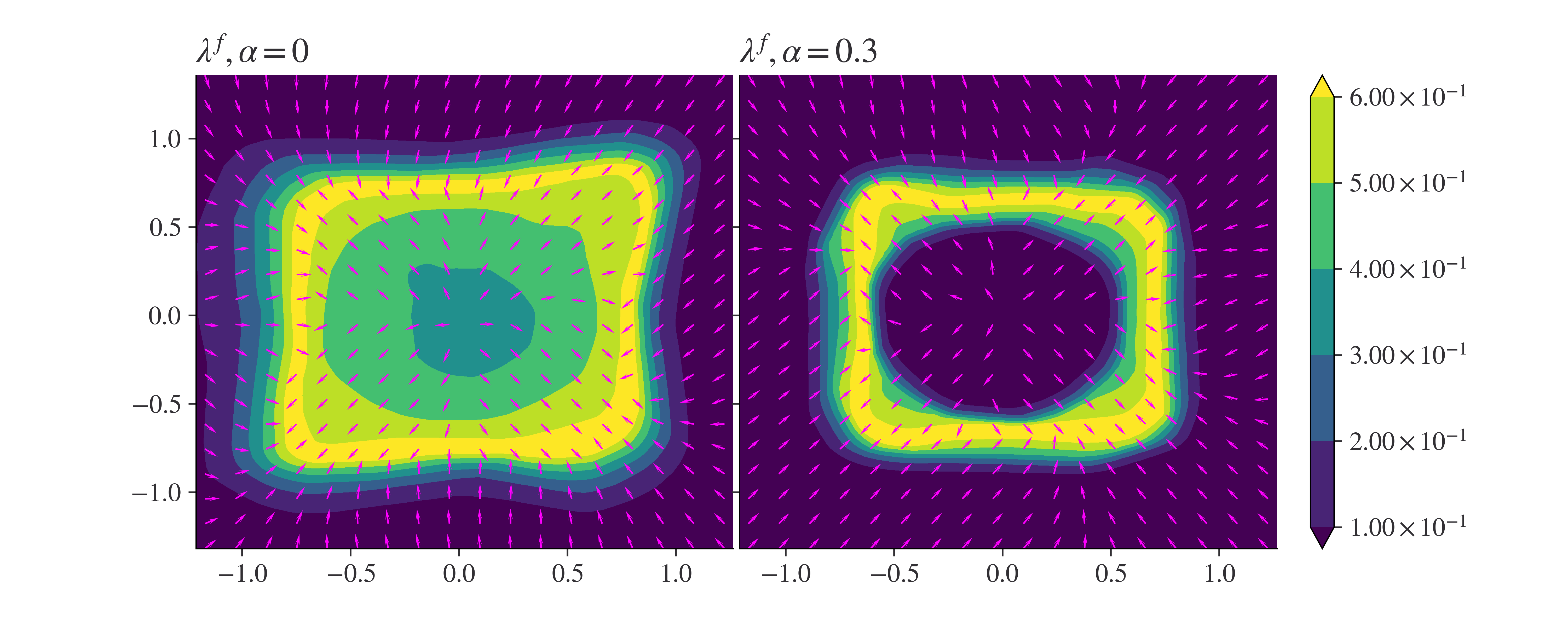}
    \caption{$\ell_{\infty}(0.2)$, Loss contour plot with learnt attack direction $\lambda^f$ compared for $\alpha=0$ (no-original data) vs $\alpha=0.3$ ($30\%$ weightage to original data). Dataset:\texttt{ circles}}
    \label{fig:alpha_compare}
\end{figure}

\subsubsection*{AdvGAN}
AdvGAN proposed by \cite{Xiao2018GeneratingNetworks} uses general adversarial networks (GANs) to generate adversarial examples. AdvGAN is the best performing black-box attack generation method according to the Madry MNIST challenge. The theoretical justifications for AdvGAN can be addressed by our framework.

For generating an attack, a GAN is trained with a generator $G$ and a discriminator $D$ optimizing the objective $\min_{G}\max_{D}\bb{E}\cl{L}_{adv}+c_1\cl{L}_{GAN}+c_2\cl{L}_{hinge}$ for some chosen constants $c_1,c_2$ determining weightage of the losses.
$\cl{L}_{GAN}=\log(1-D(G(x)))+\log(D(x))$ is the standard GAN loss (Bernoulli likelihood) for distinguishing between real and adversarially perturbed data; $\cl{L}_{adv}=-L(f(x+G(x)),y)$ is the adversarial loss maximizing the model/classifier $f$'s original loss for some $L$, e.g., MSE, cross-entropy, CW loss, etc; $\cl{L}_{hinge}=\max (0,\norm{G(x)}_2-\delta)$ is the hinge loss which softly bounds the $\ell_2$ norm of the generated attacks by $\delta$.

As $D$ appears only in $\cl{L}_{GAN}$, we can train $G$ via solving the objective $\max_{G}\bb{E}-\cl{L}_{adv}-c_1\cl{L}_{GAN}(D^*)-c_2\cl{L}_{hinge}$, where $D_G^*=\arg\max_{D}\bb{E}\cl{L}_{GAN}$.

To fit the advGAN approach into our framework, the generator $G$ corresponds to the attack function $\lambda^f$ and we choose the loss function as $\cl{L}(f(x),y)=L(f(x),y)-c_1\cl{L}_{GAN}(D_G^*)$, then the corresponding adversarial loss would be: $$\cl{L}(f,\lambda^f,\cl{Z})=\cl{L}(f(x+\lambda^f(x,y)),y)=L(f(x+\lambda^f),y)-c_1\cl{L}_{GAN}(D_G^*).$$ 
Consider the $\ell_2(\delta)$ best function attack \eqref{best_function_attack} for the above loss: 
$$\lambda^f=\arg\max_{\norm{\lambda}_2\leq \delta}\cl{L}(f(x+\lambda(x,y)),y)$$
If we consider the Lagrangian of the above objective with Lagrangian multiplier $c'_2$, we get:
$$\lambda^f=\arg\max_{{\lambda}}\cl{L}(f(x+\lambda(x,y)),y)+c'_2 (\norm{\lambda}_2-\delta).$$ 
When $\norm{\lambda_2}\leq \delta$, the optimal solution can be achieved for $c'_2=0$; when $\norm{\lambda_2}> \delta$, $c'_2$ has to be $-\infty$. So, in practice, the constrained optimization can be emulated for $\lambda^f$ by solving:
$$\lambda^f=\arg\max_{\lambda}\cl{L}(f(x+\lambda(x,y)),y)-c_2\max(0,\norm{\lambda}_2-\delta)=-\cl{L}_{adv}-c_1\cl{L}_{GAN}(D_{\lambda}^*)-c_2\cl{L}_{hinge},$$
where $c_2>0$ is a large coefficient.  This is the same loss that the generator in AdvGAN maximizes. Hence, AdvGAN can be thought of as a special case of our framework which aims to generate the best $\ell_2(\delta)$ attack. Consequently, the dynamic distillation method proposed in \cite{Xiao2018GeneratingNetworks} is equivalent to the mathematical game between $G\equiv\lambda^f$ and $f$. We can link the success of AdvGAN to the theoretical justifications provided in this work.

\subsubsection*{Conclusion}
Through this work, we were able to define a generalized framework for adversarial training, in the sense that the \textit{best}  adversarial attacks/examples \eqref{best_function_attack} can take a functional form. We showed that this functional form is essentially captured by a continuous dynamic system closely related to the gradient flow of the adversarial loss \eqref{gradient_flow_attack}. The traditional gradient-based attacks in the literature can be thought of as a discrete approximation of the same. Subsequently, the ideal attack resides in a piece-wise H\"older space and we can render neural networks to approximate this function up to arbitrary accuracy (Theorem \ref{thm:NNapprox}). This is also empirically verified through our experiments (Figure \ref{fig:gradientvsNN}). Furthermore, using neural networks to generate attacks and approximating the model with another neural network results in the adversarial training process being a mathematical game between two neural networks. We provided a convergence rate for such an adversarial game/training based w.r.t. the sample size of the data (Theorem \ref{thm:convergence}). We also compare the performance between PGD and a function-based attack generator based on neural networks; which further bolsters the argument of approximation capability of neural nets for the ideal attack. As a consequence of training a functional form, we are also able to generate attacks in a semi-white box setting, whose performance is comparable to the PGD attack. Our work provides a theoretical justification for a recent trend to use neural networks to generate attacks; additionally, we can justify the use of such networks in the context of adversarial training of a robust model. 

\bibliographystyle{siamplain}  
\bibliography{references,dataset}  

\appendix
\section{Additional definitions}
\subsection{CDS}
\begin{definition}[Continuous dynamic system]
Let $\cl{X}$ be a metric space. A mapping $T:\bb{R}^+\times \cl{X}\to \cl{X}$, $(t; x) \to T(t)x$ is a continuous ($C^0$) dynamical system if it satisfies the following:
\begin{itemize}
    \item $T(0) = I$,
\item $T(t + s) = T(t)T(s); t; s\in \bb{R}^+$
\item $T(t) \in C^0(\cl{X};\cl{X})$ for $t \geq 0$
\item $T(t)x$ is continuous in $t; x$ for $(t; x) \in [0;\infty) \times \cl{X}$
\end{itemize}
\label{cds}
\end{definition}

\begin{definition}[Attraction]
A set $A \subset \cl{X}$ attracts a set $B \subset \cl{X}$ under the dynamical system $T$ if $$\lim_{t\to \infty}dist(T(t)B,A)=0$$
\label{attraction}
\end{definition}
\begin{definition}[Point Dissipative]
A dynamical system $T$ is point dissipative if there is a bounded set $B \subset \cl{X}$ which attracts
each point of $\cl{X}$.
\label{point_dissipative}
\end{definition}
\begin{definition}[Invariant set]
A set $J\in\cl{X}$ is an invariant set of the dynamical system $T$ if $T(t)J=J\,  \forall t \in \bb{R}^+$
\label{invarient set}
\end{definition}
\begin{definition}
For a dynamic system $T$:
\begin{equation}
   \mathcal{A}(T) =\{x\in\cl{X} :T(t)x\text{ is defined and bounded }\forall t\in \bb{R}^+ \}
\end{equation}
The set $\mathcal{A}(T)$ is the compact global attractor of $T$ if it is compact, invariant
and attracts each bounded set of $\cl{X}$.
\label{globalattractor}
\end{definition}

\begin{definition}[stable and unstable sets] Given a stationary point $x_0\in \cl{X}$. The stable and unstable sets of $x_0$ are defined respectively:
\begin{align*}
    W_s(x_0)&=\{x\in\cl{X}: \lim_{t\to\infty}dist(T(t)x,x_0)=0\}\\
    W_u(x_0)&=\{x\in\cl{X}: \lim_{t\to-\infty}dist(T(t)x,x_0)=0\}
\end{align*}
\label{stable/unstable_sets}
\end{definition}

\subsection{Function classes}
Some set of definitions borrowed from \cite{Metzger2022AdversarialEstimators}.
\begin{definition}[Hölder Space]\label{def:holder}\ \\
For a function $h: \mathbb{R}^{D} \rightarrow \mathbb{R}, \partial_{d} h(x)$ is a partial derivative with respect to a $d$-th component, and $\partial^{\alpha} h:=\partial_{1}^{\alpha_{1}} \cdots \partial_{D}^{\alpha_{D}} h$ using multi-index $\alpha=\left(\alpha_{1}, \ldots, \alpha_{D}\right) .$ For $z \in \mathbb{R}$.
$\lfloor z\rfloor$ denotes the largest integer that is less than $z$. Let $\omega>0$ be a degree of smoothness. For $h:[0,1]^{D} \rightarrow \mathbb{R},$ the Hölder norm is defined as
$$
\|h\|_{\mathcal{H}\left( \omega ,[0,1]^{D}\right)}:=\max _{\alpha:\|\alpha\|_{1}<\lfloor \omega \rfloor} \sup _{x \in[0,1]^{D}}\left|\partial^{\alpha} h(x)\right|+\max _{\alpha:\|\alpha\|_{1}=\lfloor \omega \rfloor x, x^{\prime} \in[0,1]^{D}, x \neq x^{\prime}} \frac{\left|\partial^{\alpha} h(x)-\partial^{\alpha} h\left(x^{\prime}\right)\right|}{\left\|x-x^{\prime}\right\|_{\infty}^{ \omega -\lfloor \omega \rfloor}}
$$
Then, the Hölder space on $[0,1]^{D}$ is defined as
$$
\mathcal{H}\left( p ,[0,1]^{D}\right)=\left\{h \in C^{\lfloor p \rfloor}\left([0,1]^{D}\right) \mid\|h\|_{\mathcal{H}\left( p ,[0,1]^{D}\right)}<\infty\right\}
$$
Also, $\mathcal{H}\left( \omega ,[0,1]^{D}, M\right)=\left\{h \in \mathcal{H}\left( \omega ,[0,1]^{D}\right) \mid\|h\|_{\mathcal{H}\left( \omega ,[0,1]^{D}\right)} \leq M\right\}$ denotes the $M$-radius closed
ball in $\mathcal{H}\left( \omega ,[0,1]^{D}\right)$.
\end{definition}

\begin{definition}[Covering Number]\label{def:coveringnumber}\ \\ 
For some norm $\|\cdot\|$ over some metric space $\cl{X}$, the covering number $\mathcal{N}(\delta, \cl{X}, \|\cdot\|)$ is defined as the cardinality of the smallest set $C\subset\Lambda$ such that $\sup_{x\in \cl{X}}\inf_{c\in C} \|x-c\|\leq\delta$. The quantity $
\log\mathcal{N}(\delta, x, \|\cdot\|)$ is also called metric entropy.
\end{definition}

\begin{definition}[Minkowski Dimension]\label{def:minkowski}\ \\
The (upper) Minkowski dimension of a set $S \subset[0,1]^{D}$ is defined as
\[\operatorname{dim}_{M} {S}:=\inf \left\{d^{*} \geq 0 \mid \limsup _{\varepsilon \downarrow 0} \mathcal{N}(\varepsilon, {S}, \|\cdot\|_\infty) \varepsilon^{d^{*}}=0\right\}\]
where $\mathcal{N}(\varepsilon, {S}, \|\cdot\|_\infty)$ is given by Definition \ref{def:coveringnumber}.
\end{definition}

\begin{definition}[Minimum $\gamma$ covering]
\label{def:minimum_covering}
Let $S\subset \cl{X}$ then the minimum $\gamma$ covering of $S$ is define as the set of $\gamma$ radius balls $\cl{I}$ whose centres form the set $C$ such that $\card(C)=\mathcal{N}(\gamma, S, \|\cdot\|)($\ref{def:coveringnumber}). For $\ell_{\infty}$ norm, $\cup_{i=1}^{|\cl{I}|}I_i=\cl{I}$ are the set of hyper cubes of length $\gamma$ which cover $S$. If the norm is not specified, assume it to be $\ell_{\infty}$.
\end{definition}
\section{supporting results}
\begin{lemma}[Lemma A.8 \cite{Petersen2017OptimalNetworks}]
 Fix any $h \in \cl{H}(\omega, [0,1,]^D, M)$ (\ref{def:holder}) and $\bar{x} \in [0, 1]^D$.
  Let $\bar{h}(x)$ be the Taylor polynomial of degree $\lfloor{\omega}\rfloor$ of $h$ around $\bar{x}$, namely,
  \begin{equation*}
    \bar{h}(x) := \sum_{\abs{\alpha} \leq \lfloor{p}\rfloor} \frac{\partial^\alpha h(\bar{x})}{\alpha!} (x - \bar{x})^\alpha.
  \end{equation*}
  Then, $\abs{h(x) - \bar{h}(x)} \leq D^\omega M \norm{x - \bar{x}}^\omega$ holds for any $x \in [0, 1]^D$. Note that there exists some constant $\overline{C} = \overline{C}(\omega, D, M)$ such that $\sup_{\abs{\alpha} \leq \lfloor{\omega}\rfloor} \abs{\partial^\alpha h(\bar{x})/\alpha!} \leq \overline{C}M$ for $f \in \cl{H}(\omega, [0, 1]^D, M)$.
  \label{lem:taylor_approx}
\end{lemma}


\begin{lemma}[Lemma 20 \cite{Nakada2019AdaptiveDimensionality}]\label{lem:covering_division}
    Let $\cl{I}$ be a minimum $\gamma$-covering of $\cl{X}$. Then, there exists a disjoint partition $\qty{\cl{I}_i}_{i=1}^{5^D}$ of $\cl{I}$ such that $\cl{I} = \bigcup_{i=1}^{5^D} \cl{I}_i$ and $d(I_j, I_k) \geq \gamma$ hold for any $I_j \neq I_k \in \cl{I}_i$ if $\card(\cl{I}_i) \geq 2$.
\end{lemma}

\begin{lemma}[Volume of boundary cubes]
$\partial \bs{A}=\{x:x\in \bar{A_i}\bigcap\bar{A_i^c} \text{ for some } A_i\in \bs{A}\}$ be the boundary set of the pieces of any $f\in \bar{\cl{H}}(\omega,\mathcal{X})$. By definition the boundary set can be at most $D-1$ dimensions.\\
Consider the set of minimum $\gamma$ covering interacting with the boundary of the pieces,  $\mathcal{I}_{\partial}=\{I:I\in \mathcal{I}\ s.t. \, I\bigcap \partial \bs{A}\neq \phi \}$.\\
Then $\mathcal{N}(\gamma,\partial \bs{A},\norm{}_{\infty})\prec\card(\mathcal{I}_{\partial})\prec\mathcal{N}(\frac{\gamma}{2},\partial \bs{A},\norm{}_{\infty})$ and $\mu(\cl{I}_{\partial})=\cl{O}(\gamma)$.
\label{lem:volume_bndry}
\end{lemma}




\section{Proofs}
\begin{proof}[Proof of lemma \ref{lem:volume_bndry}]
For any covering $\mathcal{I}_0$ define the \textit{coverage set} as $cv(\mathcal{I}_0)=\{\bigcup I_i: I_i\in\mathcal{I}_0\}$ .
We want to show that the number of cubes $I\in \mathcal{I}$ intersecting with the boundary is controlled by the covering number of the boundary. \\
We will prove the first inequality by contradiction; assume $card(\mathcal{I}_{\partial})<\mathcal{N}(\partial \mathcal{A},\norm{}_{\infty},\gamma)$ instead. This isn't possible as $\mathcal{I}_{\partial}$ is also a covering for $\partial\mathcal{A}$ but that implies the existence of a covering smaller than the minimum covering-number of $\partial\bs{A}$ which is a contradiction. Hence, $\mathcal{N}(\partial \mathcal{A},\norm{}_{\infty},\gamma)\leq card(\mathcal{I}_{\partial})$.\\\\
For the second inequality, consider a \textit{naive} $\gamma$-covering $\mathcal{\hat{I}^\gamma}$ such that the centres of all the cubes involved in the covering are at least $\gamma$ distance apart. Let $C=\{C_{i}\}$ be the set of centers corresponding to the cubes of $\mathcal{\hat{I}}$, and they satisfy that $\min_{i\neq j} \norm{C_i-C_j}_{\infty}=\gamma$.  One can visualise it as a space-filling covering by stacking cubes adjacent to each other, such that all the cube edges pass through the centre of neighbouring cubes. By construction $card(\mathcal{I})\leq card(\mathcal{\hat{I}^\gamma})$, as $card(\mathcal{I})$ is the minimum $\gamma$-covering. \\\\
Also, let $\mathcal{\hat{I}^\gamma}_{\partial}=\{I:I\in \mathcal{\hat{I}^{\gamma}}\ s.t. \, I\bigcap \partial \bs{A}\neq \phi \}$. Notice, that $cv(\mathcal{I}_{\partial})\subseteq\{x:d(x,\partial\mathcal{A})\leq2\gamma\}$ and $cv(\mathcal{\hat{I}^{\gamma}}_{\partial})\supseteq \{x:d(x,\partial\mathcal{A})\leq\gamma\}$. This is enforced by the constraint that cubes intersecting with the boundary can have points inside it at most $2\gamma$ distance away from the boundary. Furthermore, as the radius of each cube is $\gamma$ it can at the least cover $\gamma$ distance around the boundary.
Similarly, $cv(\mathcal{\hat{I}}^{2\gamma}_{\partial})\supseteq \{x:d(x,\partial\mathcal{A})\leq2\gamma\}$.\\
Furthermore, noting that each cube of radius $2\gamma$ can be decomposed into $2^D$ smaller cubes of radius $\gamma$, we can always define some $\mathcal{I}'$ which is a set of $\gamma$-radius cubes, such that $cv(\mathcal{I}')=cv(\mathcal{\hat{I}}^{2\gamma}_{\partial})$ and $card(\mathcal{I}') =2^D card(\mathcal{\hat{I}}^{2\gamma}_{\partial})$. 
Therefore, $cv(\mathcal{I}_{\partial})\subseteq cv(\mathcal{\hat{I}}^{2\gamma}_{\partial})= cv(\mathcal{I}')$ and $\mathcal{\hat{I}}$ is the minimal-$\gamma$ covering imply that $card(\mathcal{I}_{\partial})< card(\mathcal{I}') =2^Dcard(\mathcal{\hat{I}^\gamma}_{\partial})$. (Otherwise $\mathcal{\hat{I}}\backslash\mathcal{I}_{\partial} \cup \mathcal{I}'$ is a better $\gamma$-covering than $\mathcal{\hat{I}}$).


 
Let $V$ be the minimum $\gamma/2$-covering of $\partial{A}$. 
Notice that each cube in the minimum $\gamma/2$ covering can intersect \textit{at most} $4^D$ cubes in $\mathcal{\hat{I}}^{\gamma}_{\partial}$. This is a consequence of the lattice structure induced by $\mathcal{\hat{I}}^{\gamma}_{\partial}$. (Given a $\gamma$ spacing lattice and any arbitrary point, there are at most $4^D$ lattice grid points that are within 3/2-$\gamma$ distance to the specific point.)
As, all cubes in $\mathcal{\hat{I}}^{\gamma}_{\partial}$ must intersect with some $I\in V$ by definition. Subsequently, each $I\in V$ has at most $4^D$ cubes in $\mathcal{\hat{I}}^{\gamma}_{\partial}$ that intersect with it; hence a conservative estimate for the total number of cubes in $\mathcal{\hat{I}}^{\gamma}_{\partial}$ can be bounded by $4^D\card(V)$. This implies that $card(\mathcal{\hat{I}}^{\gamma}_{\partial})\leq 4^D \mathcal{N}(\partial \mathcal{A},\norm{}_{\infty},\gamma/2)$.

Combining all the previous arguments, \begin{equation*}
\mathcal{N}(\partial \mathcal{A},\norm{}_{\infty},\gamma) \leq card(\mathcal{I}_{\partial})\leq 2^Dcard(\mathcal{\hat{I}}^{\gamma}_{\partial})\leq 2^D\cdot4^D\mathcal{N}(\partial \mathcal{A},\norm{}_{\infty},\gamma/2).
\end{equation*} 
Using the properties of covering numbers we have, $\mathcal{N}(\partial\mathcal{A},\norm{}_{\infty},\gamma),\mathcal{N}(\partial \mathcal{A},\norm{}_{\infty},\gamma/2)\asymp \frac{1}{\gamma}^{D-1}$. Consequently, $card(\mathcal{I}_{\partial})\asymp \frac{1}{\gamma}^{D-1}$.
Subsequently, $\mu(\mathcal{I}_{\partial})=\mathcal{O}(\gamma^D\times \frac{1}{\gamma}^{D-1})=\mathcal{O}(\gamma)$.
\end{proof}

\begin{proof}[Proof of Theorem \ref{thm:bestattack}]
Considering \eqref{contraint_optimisation}, let the candidate solution be $x^*=\adv=\lambda^f(x_s,y)+x_s$ where $\lambda^f$ is defined in \eqref{gradient_flow_attack}. Note that $x^*$ can't be a saddle point of $F(x)$ ( As $\dot{x}>0$ at saddle points \eqref{projected_gf_excluding_saddle_pts}) and can only be a local maxima with $\nabla{F}(x^*)=0$ or a point at the boundary of the $B_p(x_s,\delta)$ ball which is orthogonal to the surface of $B_p(x_s,\delta)$. \\
\underline{\textbf{$\ell_{\infty}$ constraint:}}

For each $i$ the gradient of the constrained function is $\nabla g_i(x)=(0,\dots,s(x-x_s)_i,\dots,0)^T$, where $s(x)$ is the sub-differential of the absolute function; $s(x)=1,-1$ for $x>0,<0$ resp. and $s(x)=[-1,1]$ for $x=0$.

It easy to see that $\nabla g_i(x)$ is just a basis for the Euclidean space. Hence, it is linearly independent and this non-linear optimisation problem falls in the LICQ class of conditions (Linear independence constraint qualification). The KKT conditions for LICQ class is \textit{necessary} for a solution.

\textbf{Stationarity:} $\nabla F(x^*)=\sum_i \mu_i^*\nabla g_i (x^*)=(\mu_1^*s (x^*-x_s)_1,\dots,\mu_i^* s(x^*-x_s)_i,\dots, \mu_D^* s(x^*-x_s)_D)^T$ where $\mu_i\geq 0$. When $x^*$ is in the interior of the $B_\infty(x_s,\delta)$ then it means the dynamics reaches equilibrium $\dot{x^*}=\nabla F(x^*)=0$, hence stationarity is satisfied for $\mu_i=0$.

When $x^*\in \partial B_\infty(x_s,\delta)$ then some of the inequality constraints are active (represented by $\bs{J}(x^*)$) and the dynamic system reaches equllibrium if and only if $\dot{x^*}=\nabla_x F(x^*)\odot(\bs{1}-\bs{J}(x^*))=0$; so the gradient will be along the diagonals or the sides of the cube determined by the active constraint. Stationary condition is again satisfied, with $\mu_i=0$ if $g_i(x^*)<0$.

\textbf{Complimentary Slackness:}
This also implies $\mu_ig_i(x)=0$ for all $i$ (for active constraint $g_i(x)=0$, for inactive constraints $\mu_i=0$). 

\textbf{Duality:}
Also note that by definition of $\pgfd{x_s}$ and continuity of the gradient; the gradient sign is consistent to the $\sign(x^*-x_s)$ enforcing duality $\mu_i\geq0$ for all $i$.
\textbf{Primal feasibility} is trivially true as our CDS is within $B_\infty(x_s,\delta)$.

\textbf{SOSC}: Given the active constraints $\bs{C}(x^*)=[(g_i(x^*))_{i\in I}]$ at $x^*$; the set of feasible directions is  $$FD=\{d\in \cl{X}:(D\bs{C}(x^*)d)_i\leq 0 \, \forall i\in I\}$$
The feasible directions in our case must be along $\bs{P}(x^*)$.
When $x^*$ is at the boundary away from saddle points, it follows a non-decreasing trajectory essentially along the feasible directions \eqref{progected_gradient_flow}. Hence, hessian along the feasible direction ($d^T\nabla_{xx}F(x)d<0$) must be negative , as the gradient along the trajectory goes from positive to zero at equilibrium ($\dot{F}(x(t))>0\to \dot{F}(x^*)=0$). The second order sufficient condition requires us to show $d^T\nabla_{xx}L(x,\mu)d>0$, where $L(x,\mu)$ is the Lagrangian, but for $\ell_{\infty}$ constraints the second derivatives corresponding to the constraints are zero hence $-d^T\nabla_{xx}L(x,\mu)d=d^T\nabla_{xx}F(x)d<0$. For $x^*$ inside the ball, all constraints are inactive and by our construction $x^*$ can't be a saddle point but it is equilibrium to an increasing trajectory hence $\nabla_{xx}F(x)$ is just a negative definite matrix.
\vspace{0.2cm}

\underline{\textbf{$\ell_{p}; 1\leq p<\infty$ constraint:}}

This is a much simpler case with only one inequality constraint $g_1(x)=\norm{x-x_s}^p_p-\delta^p$. We can again verify the KKT and SOSC to prove optimal solution. 

\textbf{Stationarity:} In this case stationarity is achieved if the gradient $\nabla F(x^*)=\mu\nabla g_1(x^*)$, essentially the gradient must be proportional to the normal vector of the $\ell_p$ ball. This is naturally satisfied as the ODE of our CDS attains equillibrium iff $\dot{x^*}=\bs{P}\nabla F(x^*)=0$; which implies that $\nabla F(x^*)$ must be orthogonal to the surface (as $\bs{P}$ is the projection matrix to $\partial B_p(x_s,\delta)$) or parallel to the normal vector $\nabla{g_1}(x^*)$.

 If $x^*$ is in the interior of $B_p(x_s,\delta)$, then the ODE at equillibria $\dot{x^*}=\bs{P}\nabla F(x^*)=\bs{I}\nabla F(x^*)=0$ implies the $\nabla{F}(x^*)=0$; hence stationarity is satisfied for $\mu=0$.
 \textbf{Complimentary slackness, duality, primal feasibility} can be verified by the same arguments used for $\ell_{\infty}$ case.

\textbf{SOSC:}
When $x^*$ is in the interior of the ball the argument is the same as $\ell_{\infty}$ constraint case. At boundary again using similar arguments we have $d^T\nabla_{xx}F(x)d<0; d\in FD$. Also, note that $\nabla_{xx}g_1(x)$ is positive definite (hessian of the norm) almost surely, so the Lagrangian $L(x,\mu)=-F(x)+\mu g_1(x)$ is also enforced to have $d^T\nabla_{xx}L(x,\mu)d>0; d\in FD$ hence the second order sufficient condition is satisfied.
With both KKT and SOSC conditions satisfied $x^*$ is the optimal solution.

The optimality of $\lambda^f(x,y)$ as a $\ell_p(\delta)$ function attack follows as $$\max\limits_{\lambda \in \cl{F}^p_{\delta}(\bpfc)}\mathbbm{E}\cl{L}(f(x+\lambda(x,y)),y)\leq \mathbbm{E}\max\limits_{\lambda \in \cl{F}^p_{\delta}(\bpfc)}\cl{L}(f(x+\lambda(x,y)),y).$$
\end{proof}

\begin{proof}[Proof of Theorem \ref{thm:NNapprox}]
We will first get an $\epsilon$ error bound for any $f\in \onephol$. Then we can extend the result to the product space for any $\vec{f}=(f^{(1)},\dots,f^{(D)})\in\phol$.

\textbf{$\epsilon$ bound for $f$}:
Consider the minimum-$\gamma$ covering $\cl{I}$ of $\cl{X}$. Using lemma \ref{lem:covering_division}, $\cl{I}=\cup_{i=1}^{5^D}\cl{I}_i$.
Let $\Xi:\cl{I}\to 2^{card (\mathcal{I})}$, $\Xi(I)=\{I'\in \mathcal{I}: I\oplus3\gamma/2\cap I'\neq \phi\}$ be a set function finding neighbouring (radius $3\gamma/2$) cubes of $I$. Consider $\tilde{\Xi}(I)=\Xi(I)\setminus\cl{I}_{\partial}$ defined for all $I\in \cl{I}\setminus{I}_{\partial}$. (where $\cubebd$ is defined in lemma 
\ref{lem:volume_bndry})
\\
Using Remark 19 of \cite{Nakada2019AdaptiveDimensionality}, one can construct a Neural network $R(\phi^{simul}_{\epsilon/2}):\bb{R}^D\to\bb{R}^{\card(\cl{I})}$, and each of its output dimension represents a neural network approximation of the Taylor-polynomial associated with $f>0$. Let $f_I$ be the Taylor-polynomial approximation of $f$ centred at $\bar{x}_I$, where $\bar{x}_{I}$ is the center of the hypercube $I$. Denote a bijective mapping $\psi:\cl{I}\to \{1,2,\dots,\card(\cl{I})\}$ which maps any hypercube $I\in\cl{I}$ to the associated output dimension of $R(\phi^{simul}_{\epsilon/2})$. Then $R(\phi^{simul}_{\epsilon/2})_{\psi(I)}$ is an $\epsilon/2$ approximation of $f_{I}$. 

For any general $f$, consider a shifted version of the function which is always positive  $f_1=f+M+1$ where $M$ is the upper bound  of $|f|$ based on H\"older norm (as $R(\phi^{simul}_{\epsilon/2})$ is valid for approximation of non-negative functions).
For a cube $I\in\cubebd$,  we define a constant zero neural network $R_0(x)$ with all weights and bias being zeros.
For cubes $I\in \cl{I}\setminus\cubebd$, we define the neural network approximation to the Taylor polynomial around the center of the cube as $R(\phi^{simul}_{\epsilon/2})_{\psi(I)}$. The piece-wise H\"older function $f\in\onephol$ restricted to any of the hypercubes $I\in \cl{I}\setminus\cubebd$, will belong to $\cl{H}(\omega,I)$ (\ref{def:holder}). In terms of architecture requirements from Remark 19 of \cite{Nakada2019AdaptiveDimensionality}, we have $W=C_1\epsilon^{-d/\omega},w=C_2\epsilon^{-d/\omega},\kappa=\cl{O}(\epsilon^{-s}),L=C_3$ for $R(\phi^{simul}_{\epsilon/2})$. Note that these neural networks approximate the Taylor polynomial of $f_1$ at the center of the cube, for all $x\in I$ with $\epsilon/2$ accuracy. And outside the cube it is strictly less than the polynomial it is approximating and 0 after a lag of $\gamma/2$.

As the sum and max functions can be replicated by the Relu neural networks, we can then mimic the proof of Theorem 5  for the piece-wise H\"older case \cite{Nakada2019AdaptiveDimensionality}, as follows.
Consider the neural network realisation of $$R(\phi_{\epsilon}^{f_1})=\max_{i\in 5^D }\{\sum_{I\in \mathcal{I}_i\setminus\cubebd}R(\phi^{simul}_{\epsilon/2})_{\psi(I)}+R_0(x)\},$$
hence for any $x\in I$ for some $I\in\mathcal{I}\setminus\cubebd$ we have 
\begin{align*}
    \abs{R(\phi_{\epsilon}^{f_1})-f_1}&=\max\{\max_{I'\in \tilde{\Xi}(I)} R(\phi^{simul}_{\epsilon/2})_{\psi(I')}-f_1(x),f_1-\max_{I'\in \tilde{\Xi}(I)} R(\phi^{simul}_{\epsilon/2})_{\psi(I')}\}\\
    &\leq \max_{I'\in \tilde{\Xi}{I}} \abs{R(\phi^{simul}_{\epsilon/2})_{\psi(I')}-f_1(x)}\\
     &\leq \max_{I'\in \tilde{\Xi}{I}} \abs{R(\phi^{simul}_{\epsilon/2})_{\psi(I')}-f_{I'}(x)}+\max_{I'\in \tilde{\Xi}{I}}\abs{f_{I'}(x)-f_1(x)}\\
     &\leq \epsilon/2 +D^{\omega}M(\frac{3\gamma}{2})^{\omega},
\end{align*}
where the last inequality uses Remark 19 \cite{Nakada2019AdaptiveDimensionality} and lemma \ref{lem:taylor_approx}.
For side-length of the cube $\gamma=D^{-1}(3M)^{\frac{-1}{\omega}}\epsilon^{\frac{1}{\omega}}$, we then have $\abs{R(\phi_{\epsilon}^{f_1})-f_1}\leq \epsilon$\\
The above approximation works for all cubes upto the measure of the cubes in $\cubebd$. As we are working with probability measures the $\mu(\cubebd)= \mathcal{O}(Vol(\cubebd))=\mathcal{O}(\gamma)=\mathcal{O}(\epsilon^{\frac{1}{\omega}})$ is a consequence of lemma \ref{lem:volume_bndry}.

Note that the composition of the sum and maximum function as neural network only increases $L,W,w$ by some constant depending on $\card(\cl{I})$ and $5^D$.
 The number of layers $L$, number of non-zero weights $W$ and in turn the width $w$ are additive w.r.t composition of neural networks. Also, $\kappa$ is the maximum of the $\kappa's$ of the constituent neural nets being composed. $\kappa$ for the sum and and max neural network are just some constants dependent on $M$ and independent of $\epsilon$. For small enough $\epsilon$ architecture requirement for $R(\phi_{\epsilon}^{f_1})$ is same as $R(\phi^{simul}_{\epsilon/2})$ up to constants independent of $\epsilon$; $W=C_1'\epsilon^{-d/\omega},w=C_2'\epsilon^{-d/\omega},\kappa=\cl{O}(\epsilon^{-s_1}),L=C_3'$. Similarly, one can  get a neural network realisation of $f$ by taking an affine transformation  $R(\phi_{\epsilon}^f)=R(\phi_{\epsilon}^{f_1})-M-1$ which adds constant additional layers and constant number of non-zero weights.
 
 In conclusion, $\epsilon$ bound for $f_1$ above leads to $\epsilon$ bound for $f$:
 $$\abs{R(\phi_{\epsilon}^{f})-f}=\abs{R(\phi_{\epsilon}^{f_1})-f_1}\leq \epsilon.$$
\textbf{$\epsilon$ bound for $\vec{f}$}:
Consider a concatenation of neural nets $R(\phi_{\epsilon}^{\vec{f}})=[R(\phi_{\epsilon}^{f^{(1)}}),\dots, R(\phi_{\epsilon}^{f^{(D)}})]$; where $R(\phi_{\epsilon}^{f^{(i)}})$ is an $\epsilon$- neural network approximation of $f^{(i)}\in \bar{\cl{H}}(\omega_i,\cl{X})$ proved previously. \\It is easy to see that the architecture requirements  for the $R(\phi_{\epsilon}^{\vec{f}})$ are $W=\sum C_1^{(i)}\epsilon^{-d/\omega_i}=\cl{O}(\epsilon^{-d/\omega^*}); w=\sum C_2^{(i)}\epsilon^{-d/\omega_i}=\cl{O}(\epsilon^{-d/\omega^*}); \kappa=\cl{O}(\max(\epsilon^{-s_{1}},\dots,\epsilon^{-s_{d}}))=\cl{O}(\epsilon^{-s^*}); L=\sum C_3^{(i)}=C_3^*$. Where $\omega^*=\min\{\omega_i\}, s^*=\max\{s_i\}$.\\
Corresponding to each $R(\phi_{\epsilon}^{f^{(i)}})$ there is a measurable set  $\nu_i$ with $\mu(\nu_i)=\cl{O}(\epsilon^{1/\omega_i})$ such that:
$$\sup_{x\in \cl{X}_s\setminus \nu_i}\abs{R(\phi_{\epsilon}^{f^{(i)}})(x)-f^{(i)}(x)}\leq \epsilon.$$
This implies for $\nu=\cup_i \nu_i$; $\mu(\nu)=\cl{O}(\epsilon^{1/\omega^*})$ and the suitable architecture  requirements:
\begin{align*}
    \sup_{x\in \cl{X}_s\setminus\nu}||R(\phi_{\epsilon}^{\vec{f}})(x)-\vec{f}(x)||_{\infty}=\sup_{x\in \cl{X}_s\setminus\nu}\sup_{i\in\{1,\dots,D\}} \abs{R(\phi_{\epsilon}^{f^{(i)}})(x)-f^{(i)}(x)}\leq \epsilon.
\end{align*}
\end{proof}
\begin{proof}[Proof of lemma \ref{lem:motivation piece wise holder}]
   WLOG, $x\in W_s(x_0)$ for some $x_0\in E$. Then by definition, $\lambda^{pgf}(x,y)\in W_s(x_0)$.
    $$\lambda^{pgf}(x,y) = \lim_{t\to\infty}\pgfd{x}(t)x-x=\int_0^{\infty}\underbrace{\bs{P}(X(t,x))}_{h_1(t,x)}[\underbrace{\nabla F(X(t,x))+\sum_{\eta\in S}\psi_{\epsilon}^{\eta}(X(t,x))\nu^{\eta}_{max}}_{h_2(t,x)}]\, dt-x$$
    Where, $X(t,x)=x+z(t,x)$ ; $\dot{z}(t,x)=h_1\cdot h_2(t,x)$ and $z(0)=0$. Let, $t_b(x)= \inf\{t: X(t,x)\in\partial B_p(x,\delta)\}$. \\
    Before time $t_b$ constraints are inactive and $\bs{P}(X(t,x))=\bs{I}$.
    \begin{align*}
        \lambda^{pgf}(x,y)&=\underbrace{\int_0^{t_b(x)}h_2(t,x)\, dt}_{I_0}+\underbrace{\int_{t_b(x)}^{\infty}h_1.h_2 (t,x)\,dt}_{I_1} -x\\
        &=X(t_b,x)+X(\infty,\underbrace{x_1}_{X(t_b,x)}) -x
    \end{align*}
    As, $\nabla F (X),\psi_{\epsilon}^{\eta}(X)$ are at least $C^k$ smooth in $X$, using Theorem 3.1 (Differentiation w.r.t. initial condition) of \cite{Graef2018OrdinaryProblems} inductively, we get $I_0=X(t_b,x)$ is $C^k$ smooth w.r.t. $x$.
    For $t>t_b$ and even $p<\infty$, note that $P(X)$ is $C^{\infty}$ smooth w.r.t. $X$ because $\norm{X-x}_p^p-\delta^p$ is $C^{\infty}$ w.r.t $X$. Hence, using Theorem 3.1\cite{Graef2018OrdinaryProblems} inductively we get $I_1$ is $C^k$ smooth w.r.t. $x_1=X(t_b,x)$. Using the chain rule, then we have $I_1$ is $C^k$ smooth w.r.t. $x$. For, even $p<\infty$, $\lambda^{pgf}(x,y)$ is then $C^k$ smooth w.r.t. $x$.\\
    In the case $p=\infty$ or odd, for $t>t_b$, $\bs{P}(X)$ is only smooth where $\sign(X(t,x)-x)$ is consistent. Let, $t_1,t_2,\dots, t_m$ be the times where the sign changes and $t_{m+1}=\infty$. Then,
    \begin{align*}
         \lambda^{pgf}(x,y)&=\underbrace{\int_0^{t_b}h_2(t,x)\, dt}_{I_0}+\underbrace{\int_{t_b}^{t_1}h_1.h_2 (t,x)\,dt}_{I_1} +\sum_{l=2}^{m+1} \underbrace{\int_{t_{l-1}}^{t_{l}}h_1.h_2 (t,x)\,dt}_{I_l} -x\\
         &=X(t_b,x)+X(t_1-t_b,\underbrace{x_1}_{X(t_b,x)})+\sum_{l=2}^{m+1}X(t_l-t_{l-1},\underbrace{x_l}_{X(t_{l-1},x)})-x
    \end{align*}
Using similar arguments as before we know that, each $I_l$ is $C^k$ smooth w.r.t. $x_l$. Using induction:\\
 \textbf{Base step}: $I_0=x_1$ is $C^k$ smooth w.r.t. $x$, $I_1$ is $C^k$ smooth w.r.t. $x_1$ hence using chain rule $I_1$ is $C^k$ smooth w.r.t. $x$. In turn $X(t_1,x)$ is $C^k$ smooth w.r.t. $x$\\
 \textbf{Induction Step:} If $x_l=X(t_{l-1},x)$ is $C^k$ smooth then $X(t_l,x)$ is $C^k$ smooth.\\
 $x_{l+1}=X(t_l,x)=X(t_{l-1},x)+I_l$, but $I_l$ is $C^k$ smooth w.r.t. $X(t_{l-1},x)$, hence using chain rule $I_l$ and consequently $x_{l+1}$ is $C^k$ smooth w.r.t. $x$.

\end{proof}

\begin{proof}[Proof of Theorem \ref{thm:convergence}]
The proof relies on the application of theorem 3.1 \cite{Metzger2022AdversarialEstimators} and almost identical to the proof  of their theorem 3.2 with the major difference being the function classes considered in our framework. This results in change of $\epsilon_n,\Bar\epsilon_n$ terms which correspond to the expected difference of the adversarial loss evaluated at the sieve projection and the true parameters. In our case the sieve projection is just the neural network approximation using $n$ samples. Using assumption $A3$ and corollary \ref{cor:NN-approx} we get: \begin{align*}
    \epsilon_n =\sup_{f\in\cl{C}_n}\mathbb{E}[\cl{L}(f,\lambda^f, \cl{Z})-\cl{L}(f,\pi_n\lambda^f, \cl{Z})]&\prec \sup_{f\in\cl{C}_n} \|\pi_n \lambda^f-\lambda^f\|_{\widetilde{{\mathcal{X}}}}^2 + \mathbb{P}^{2\omega^*}(x\not\in\widetilde{{\mathcal{X}}})\\
    &\prec \min(W_n,w_n,D^{-1/p}\delta)^2+\min{(W_n,w_n)}^{-2\omega^*/d^*}\\
     &\prec n^{-2\omega^*\underline r/d^*(\underline r+2)}=n^{-2/(\underline r+2)}
\end{align*}  
where $\mathbb{P}(x\not\in\widetilde{{\mathcal{X}}})$ corresponds to the small measurable set where the neural network can't approximate (the volume which covers the pieces of our piece-wise H\"older functions). Also, $D^{1/p}\delta$ is a hard bound for the function class $\pfc$, but when considering very wide networks for approximation, that bound is never exceeded. \\
Similarly we can bound $\Bar{\epsilon_n}$ using $A2$ and Corollary \ref{cor:NN-approx}:
\begin{align*}
    \Bar\epsilon_n=\mathbb{E}[\cl{L}(\pi_nf^*,\lambda^{f^*}, \cl{Z})-\sup_{f\in\cl{C}_n}\cl{L}(f^*,\lambda^{f^*}, \cl{Z})]&\prec \|\pi_n f^*-f^*\|_{\widetilde{{\mathcal{X}}}}^2\\
    &\prec \min(W_n,w_n)^2\prec  n^{-2\omega^*\underline r/d^*(\underline r+2)}=n^{-2/(\underline r+2)}
\end{align*}
The rest of the proof follows using arguments of theorem 3.2 \cite{Metzger2022AdversarialEstimators}.
\end{proof}

\section{Additional experiments and details}
\label{A:exp_details}
\begin{figure}[h]
    \centering
    \includegraphics[scale=0.7]{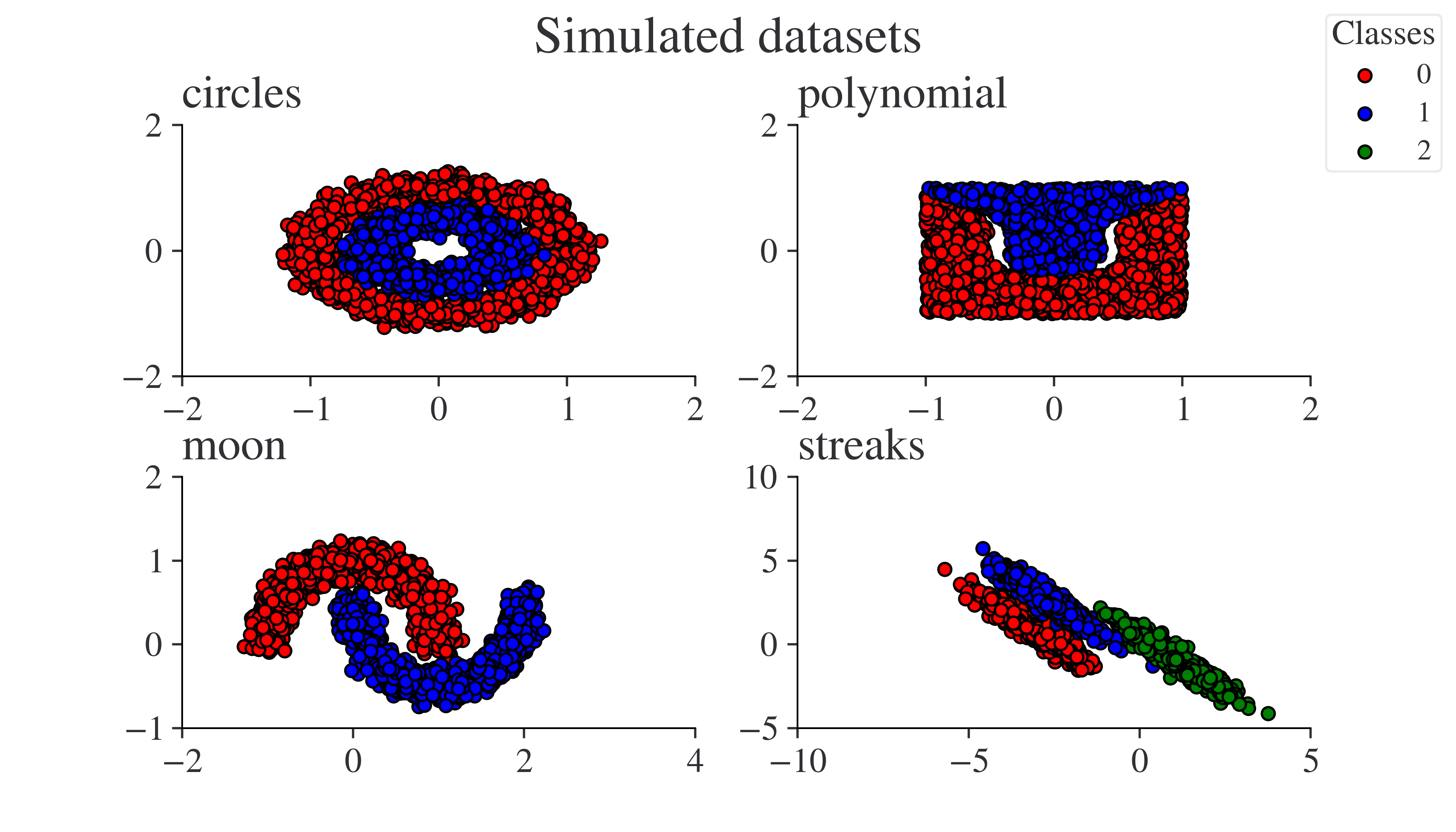}
    \caption{Simulated 2D datasets for classification setting.}
    \label{fig:dataset_classification}
\end{figure}

Code for simulated dataset:
\url{https://github.com/rhaldarpurdue/Adversarial-NN-simulations}

\subsection{Architecture details}
We use Pytorch libraries for all of our implementations. For optimization, we use the Adam method proposed by \cite{Kingma2014Adam:Optimization}.
\subsubsection*{Classification:}
$D=dim(x)$ and $C$ is the total number of classes; $c\in \{1,\dots,C\}$ is the label index for each sample. The attack generator $\lambda^f(x,y)\coloneqq S_c(x)\odot \cl{D}_c(\cl{E}(x))$ if $y=c$. \\
$\cl{E},\cl{D}_c,S_c$ are the encoder, decoder, scaler resp.. The scaler is responsible for the magnitude of the attack. (Table \ref{tab:arch1})
\begin{table}[h]
\small
\centering
\caption{Pytorch Architecture for classification setting}
\label{tab:arch1}
\begin{tabular}{|c|l|l|l|} 
\hline
\multirow{2}{*}{$f$}                                                                                                                                                                 & \multicolumn{3}{c|}{$\lambda^f$}                                                                                                                                                                                                                                                                                                                                                                          \\ 
\cline{2-4}
                                                                                                                                                                                     & $\cl{E}$                                                                                      & $\cl{D}_c$                                                                                                                                                                                                     & $S_c$                                                                                              \\ 
\hline
\multicolumn{1}{|l|}{\begin{tabular}[c]{@{}l@{}}Linear(2, 50)\\LeakyReLU(0.01),\\Linear(50,100),\\LeakyReLU(0.01),\\Linear(100,15),\\LeakyReLU(0.01),\\Linear(15, $C$)\end{tabular}} & \begin{tabular}[c]{@{}l@{}}Linear(2, 50),\\LeakyReLU(0.01),\\Linear(50,100)\end{tabular} & \begin{tabular}[c]{@{}l@{}}LeakyReLU(0.01),\\Linear(100,50),\\LeakyReLU(0.01),\\Linear(50,15),\\LeakyReLU(0.01),\\Linear(15,2),\\$\sqrt{D}\delta\cdot$L2Normalise(),\\Clamp($-\delta,\delta$) For~$\ell_{\infty}$\end{tabular} & \begin{tabular}[c]{@{}l@{}}Linear(2,20),\\LeakyReLU(0.01),\\Linear(20,1),\\Sigmoid()\end{tabular}  \\
\hline
\end{tabular}
\end{table}

\texttt{L2Normalise()} is a user defined function which projects the input into a unit circle.
\subsubsection*{Regression:}
In similar spirit, we construct the attack network for regression setting as $\lambda^f(x)\coloneqq S(x)\odot H(x)$. (Table \ref{tab:arch2})
\begin{table}[h]
\small
\centering
\caption{Pytorch Architecture for regression setting}
\label{tab:arch2}
\begin{tabular}{|c|l|l|} 
\hline
\multirow{2}{*}{$f$}                                                                                                                                & \multicolumn{2}{c|}{$\lambda^f$}                                                                                                                                                                                                                                \\ 
\cline{2-3}
                                                                                                                                                    & $H$                                                                                                                                                        & $S$                                                                                                \\ 
\hline
\multicolumn{1}{|l|}{\begin{tabular}[c]{@{}l@{}}Linear($D$,50),\\LeakyReLU(0.01),\\Linear(50,20),\\ nn.LeakyReLU(0.01),\\Linear(20,1)\end{tabular}} & \begin{tabular}[c]{@{}l@{}}Linear($D$,50),\\LeakyReLU(0.01),\\Linear(50,50),\\LeakyReLU(0.01),\\Linear(50,$D$),\\$\delta\cdot$L2Normalise()~~\end{tabular} & \begin{tabular}[c]{@{}l@{}}Linear(2,20),\\LeakyReLU(0.01),\\Linear(20,1),\\Sigmoid()\end{tabular}  \\
\hline
\end{tabular}
\end{table}

\subsubsection*{Details for $f_{PGD},\lambda_{PGD}$}
For both regression and classification settings $f_{PGD}$ uses the same architecture as $f$ mentioned in Tables \ref{tab:arch1},\ref{tab:arch2}.\\
$\lambda_{PGD}$ is generated using \eqref{pgd_attack}, with $\gamma=0.01$ with $T=50$ the total number of iterations/steps. Also, for each sample we try multiple initialisation, then we choose the best performing attack for each sample based on which one has the maximum loss. The number of such initialisation or restarts is set to $10$ in our experiments.
\subsection{Stabilising effect of $\alpha$}
\begin{figure}[h]
    \centering
    \includegraphics[scale=0.5]{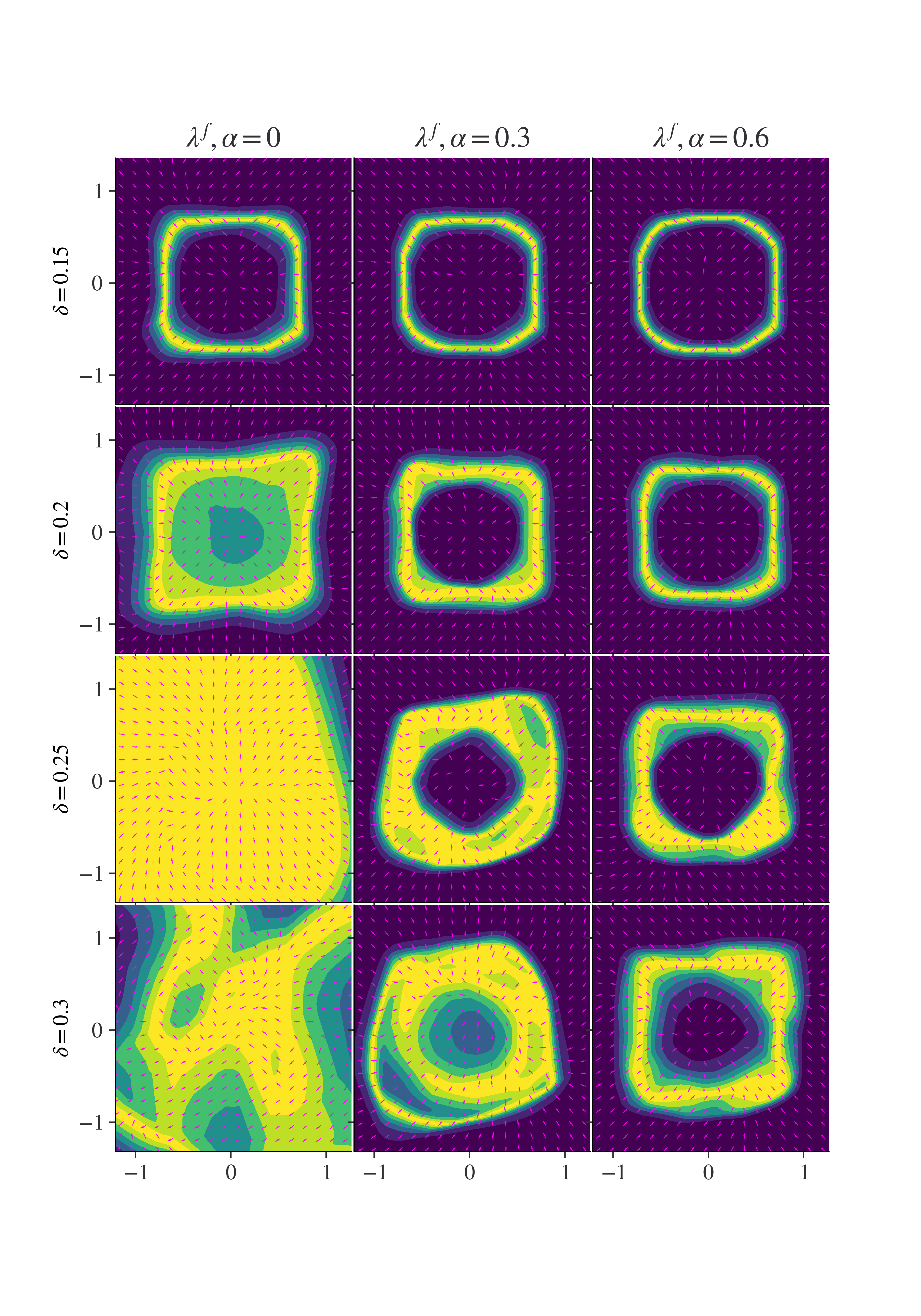}
    \caption{$\ell_{\infty}(\delta)$, Loss contour plot with learnt attack direction $\lambda^f$ compared for $\alpha=0$ (no-original data) vs $\alpha=0.3,0.6$ ($30\%,60\%$ weightage to original data). Dataset:\texttt{ circles}. Even for unfeasible $\delta\geq 0.25$ (attack strength so large that samples cross the true decision boundary after perturbation), introducing $\alpha\neq0$ helps stabilise the EDB}
    \label{fig:alpha_for_various_delta}
\end{figure}
\end{document}